\documentclass{article}
\usepackage{etoolbox}

\usepackage{amsmath,amsthm}
\usepackage[dvipsnames,table]{xcolor}

\newtoggle{arxiv}
\toggletrue{arxiv}
\iftoggle{arxiv}{
  \usepackage[parfill]{parskip}
  \usepackage[citestyle=numeric-comp,sorting=nyt,maxbibnames=99,backend=biber]{biblatex}
  \newcommand{\citep}{\parencite}
  \newcommand{\citet}{\textcite}
  \addbibresource{biblio.bib}

  \setlength{\textwidth}{6.8in}
  \setlength{\textheight}{9in}
  \setlength{\oddsidemargin}{0in}
  \setlength{\evensidemargin}{0in}
  \setlength{\topmargin}{-0.5in}
  \setlength{\marginparwidth}{0.8in}
  \setlength{\parskip}{6pt}
  \setlength{\parindent}{0pt}

  \RequirePackage[T1]{fontenc}
  \RequirePackage[tt=false, type1=true]{libertine}
  \RequirePackage[varqu]{zi4}
  \RequirePackage[libertine]{newtxmath}

}{
  \usepackage{amsfonts}
  \addtolength{\tabcolsep}{-3pt}
  \def\setstretch#1{\renewcommand{\baselinestretch}{#1}}
  \setstretch{0.98}
  \addtolength{\parskip}{-1pt}

  \usepackage[compact]{titlesec}
  \titlespacing{\section}{0pt}{*1}{*0}
  \titlespacing{\subsection}{0pt}{*1}{*0}
  \usepackage[subtle, mathdisplays=tight, charwidths=tight, leading=normal]{savetrees}
  \addtolength\textfloatsep{-0.5em}
  \addtolength\intextsep{-0.2em}
}

\usepackage{bm}

\usepackage{booktabs}
\usepackage{microtype}
\usepackage{multirow}
\usepackage[inline]{enumitem}
\usepackage{graphicx}
\usepackage[font=small]{caption}
\usepackage{algorithm,algorithmicx,algpseudocode}
\usepackage{afterpage}
\usepackage{float}

\usepackage{url}
\usepackage{hyperref}
\definecolor{link}{HTML}{307ef1}
\hypersetup{
     colorlinks   = true,
     linkcolor    = link,
     citecolor    = link,
     urlcolor     = link,
     pdfborderstyle={/S/U/W 1},
}

\usepackage[capitalise,noabbrev]{cleveref}

\usepackage{xspace}
\let\wrapperincludegraphics\includegraphics
\newcommand{\github}{\raisebox{-1.3pt}{\wrapperincludegraphics[height=1.05em]{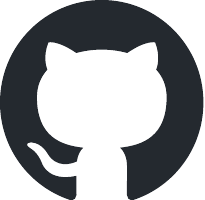}}\xspace}
\newcommand{\hf}{\raisebox{-1.3pt}{\wrapperincludegraphics[height=1.05em]{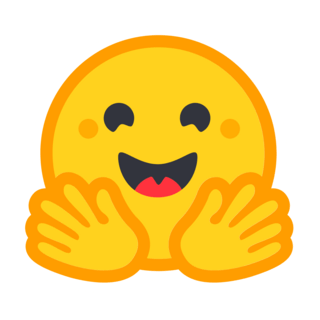}}\xspace}

\newtheorem{theorem}{Theorem}
\newtheorem{lemma}{Lemma}[section]
\newtheorem{corollary}[lemma]{Corollary}
\newtheorem{proposition}[theorem]{Proposition}

\def\*#1{\boldsymbol{#1}}

\definecolor{betterGreen}{HTML}{0B8043}
\definecolor{worseRed}{HTML}{C5221F}

\definecolor{tokDColor}{HTML}{8FBF7A}
\definecolor{tokMColor}{HTML}{D9A441}

\newcommand{\tabfont}{\small}

\newcommand{\dgain}[1]{\,{\scriptsize\textcolor{betterGreen}{$(+#1)$}}}
\newcommand{\dloss}[1]{\,{\scriptsize\textcolor{worseRed}{$(#1)$}}}
\newcommand{\dnil}{\,{\scriptsize\textcolor{black!50}{$(0.0)$}}}
\iftoggle{arxiv}{
  \title{Multi-Token Residual Prediction}
  \author{%
    Yufeng Xu$^{1,2}$ \quad Zishuo Bao$^{1,2}$ \quad Qian Wang$^{2}$ \quad Zeshen Zhang$^{1}$ \quad Haoqi Zhang$^{2}$ \\[0.3em]
    Bowen Peng$^{3}$ \quad Ang Li$^{1,4}$ \quad Rahul Chalamala$^{4}$ \quad Yucheng Lu$^{1,2}$ \\[0.6em]
    $^{1}$New York University \quad $^{2}$New York University Shanghai \quad $^{3}$Nous Research \quad $^{4}$Modal \\[0.3em]
    \texttt{\{xu.yufeng, lu.yucheng\}@nyu.edu} \\[0.4em]
    \github \href{https://github.com/heavyball-research/multi-token-residual-prediction}{\textcolor{link}{\textbf{Code}}} \quad
    \hf \href{https://huggingface.co/collections/heavyball/sdar-mrp}{\textcolor{link}{\textbf{Models}}}
  }
  \date{}
}{
  \title{Multi-Token Residual Prediction}
}

\begin{document}

\maketitle

\renewcommand{\thefootnote}{\arabic{footnote}}
\begingroup\renewcommand{\thefootnote}{}\endgroup

\begin{abstract}
Diffusion Language Models (DLMs) generate text by iteratively denoising masked token sequences, offering a tradeoff between parallelism and quality compared to autoregressive models. In current practice, the number of tokens decoded per step is controlled by a confidence threshold, and quality degrades monotonically as more tokens are denoised per step. We introduce \textbf{Multi-token Residual Prediction (MRP)}, a lightweight module that enables dependency-aware multi-token denoising within a single backbone forward pass. MRP exploits a key property of the denoising process: the logit distributions at adjacent denoising steps are remarkably similar. Rather than running the backbone a second time to obtain the next-step logits, MRP predicts the residual between steps from the backbone's hidden states, effectively denoising more tokens per backbone forward at a fraction of the cost. We apply MRP across the two operating regimes of DLM decoding. In the \emph{high-quality-low-throughput} static denoising regime, MRP serves as a drafter for speculative decoding: its proposals are verified against the backbone, yielding lossless acceleration of up to $1.4\times$ in SGLang. In the \emph{low-quality-high-throughput} dynamic denoising regime, MRP instead drives a remasking scheme that revokes over-eager reveals, recovering most of the accuracy lost to aggressive low-threshold decoding and improving accuracy by up to $22.6$ points on code generation task HumanEval and $17.7$ points on reasoning task GSM8K.
\end{abstract}

\section{Introduction}
\label{sec:intro}

\afterpage{%
\begin{figure}[!htbp]
    \centering
    \includegraphics[width=1.0\linewidth]{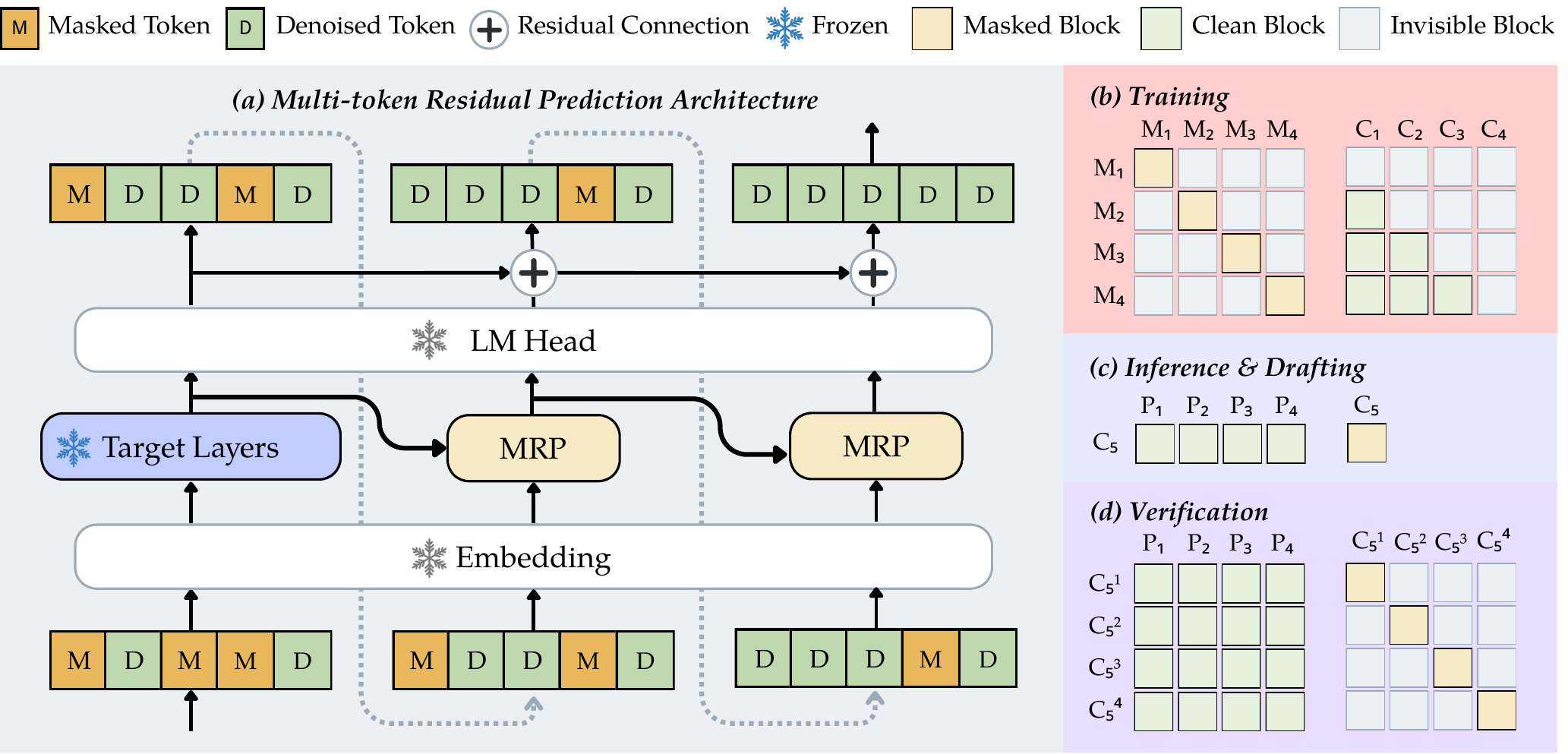}
    \caption{The figure above illustrates the training and inference pipeline of our
    multi-token residual prediction (MRP) method. \textit{(a) Architecture.} A frozen
    DLM backbone embeds the input sequence, target layers produce hidden states for the
    masked positions, and a chain of lightweight MRP heads sequentially predict the
    residual of the next logit state and decode the next probable token \textcolor{tokDColor}{\textbf{[D]}} at the
    adjacent masked \textcolor{tokMColor}{\textbf{[M]}} position through the shared LM head. \textit{(b) Training.} The
    sequence is duplicated and one half is randomly masked; loss is computed only on the
    masked half, where each masked block $M_i$ attends to itself and to all preceding
    clean blocks $C_{1{:}i-1}$. \textit{(c) Inference \& drafting.} A single new block
    is decoded in parallel given the committed clean prefix $P_{1{:}i-1}$, supporting
    both direct and speculative generation under an identical attention pattern.
    \textit{(d) Verification.} The $K{+}1$ draft candidates $C_i^{0{:}K}$ are verified
    jointly against the backbone, where candidate $C_i^{k}$ embeds the first $k$ draft
    tokens and attends only to the clean prefix $P_{1{:}i-1}$ and to itself, so
    all candidates are scored in one forward pass.}
    \label{fig:1}
\end{figure}
}

The design of language models involves a fundamental tension between \emph{sequential dependency} and \emph{parallelism}. Autoregressive models resolve this tension in favor of dependencies: each token is generated conditioned on the full history of preceding tokens, yielding high-quality text at the cost of strictly sequential decoding~\citep{brown2020GPT3, liu2024DeepSeekV3, yang2025qwen3}. Diffusion Language Models (DLMs) make the opposite choice, predicting all the subsequent tokens simultaneously and refining them through iterative denoising~\citep{sahoo2024MDLM, lou2024SEDD, nie2025LLaDA, hu2026RCDLM}. This enables significant parallelism, but often at the cost of generation quality.

Recent efforts have sought to interpolate between these two extremes. Block diffusion~\citep{arriola2025BD3LM} applies diffusion within fixed-length blocks while maintaining autoregressive ordering across blocks. Threshold-based decoding~\citep{cheng2025SDAR, wu2025FastdLLMv2, wu2025FastdLLM, hong2025WINO} adaptively controls how many tokens to unmask per step according to model confidence.
These methods offer practical improvements, but they operate along the same quality--throughput Pareto frontier: unmasking additional tokens in a single step increases throughput at the expense of quality, precisely because each token is decoded without awareness of the other tokens resolved in the same step. The central question thus remains open: \emph{can we decode multiple tokens per step in a DLM while preserving the inter-token dependencies that are essential for generation quality?}

We answer this question affirmatively with a key insight: the logit distributions produced by a diffusion backbone at adjacent denoising steps are remarkably similar. When only a few tokens are unmasked between two consecutive steps, the backbone's output barely changes, yet computing it requires a full forward pass. The residual between the two outputs is significantly smaller in magnitude than the outputs themselves, making it a far easier prediction target: the zero function is already a reasonable approximation, and a small module only needs to learn corrections to an already-good prediction, following the same principle underlying residual learning~\citep{he2016ResNet}.

Building on this observation, we introduce \textbf{Multi-token Residual Prediction (MRP)}, a lightweight module that predicts the inter-step logit residual directly from the backbone's hidden states. At inference time, a single backbone forward pass produces both the current logits and the hidden representation from which MRP infers the residual correction. The corrected logits approximate the output of a second backbone forward pass at a fraction of the computational cost, enabling additional tokens to be decoded without running the backbone again.

We deploy MRP across the two operating regimes of DLM decoding. In the \emph{high-quality, low-throughput} static denoising regime, where output quality must be preserved, MRP serves as a lightweight drafter for \textbf{speculative decoding}: the MRP module proposes additional tokens beyond what the backbone's confidence yields, and these proposals are verified by the backbone in its subsequent forward pass. Accepted tokens incur no quality loss, yielding lossless acceleration. 
As an additional knob, MRP also supports \textbf{direct decoding}, in which the corrected log-density is used without verification, providing a tunable quality--speed tradeoff that is strictly more favorable than threshold-based decoding at matched throughput.
In the \emph{low-quality, high-throughput} dynamic denoising regime, where a low confidence threshold unmasks many tokens per step at the expense of accuracy, MRP instead drives a \textbf{remasking decoding} scheme: the corrected log-density revokes over-eager reveals, recovering most of the accuracy lost to aggressive low-threshold decoding at a moderate efficiency cost.

Our contributions are as follows:
\begin{itemize}[nosep, leftmargin=12pt]
    \item We identify and theoretically characterize the inter-step log-density residual in masked diffusion language models (Section~\ref{sec:method}). This characterization motivates and justifies the design of a lightweight residual predictor.
    \item We introduce MRP, a residual prediction module trained with a tailored objective, and show how it applies across both regimes of DLM decoding: speculative decoding in the static regime and remasking in the dynamic regime, analyzing the operating point of each (Section~\ref{sec:method}).
    \item We integrate our method into the open-sourced inference engine SGLang \citep{zheng2024sglang} and demonstrate that MRP achieves up to $1.4\times$ lossless speedup in the static regime and recovers up to $22.6$ accuracy points in the dynamic regime, on SDAR models across three scales and multiple benchmarks (Section~\ref{sec:experiments}).
\end{itemize}

\section{Preliminary}
\label{sec:prelim}

We start by describing the decoding procedure of Diffusion Language Models (DLMs) and establish notation used throughout the paper.

\textbf{Denoising Loop.} A DLM generates a sequence of $L$ tokens by iteratively reversing a masking process~\citep{sahoo2024MDLM, nie2025LLaDA}. Generation begins from a fully masked sequence $\*x_T$ in which every position is set to a special \texttt{[MASK]} token. At each denoising step $t = T, T{-}1, \ldots, 1$, a bidirectional Transformer $f$ (the \emph{backbone}) takes the partially masked sequence $\*x_t$ as input and produces logits $\*\ell_t = f(\*x_t) \in \mathbb{R}^{L \times V}$, where $V$ is the vocabulary size. These logits are unnormalized log-densities; a softmax over the vocabulary normalizes them into the per-position predictive distribution $\*\pi_t^i = \mathrm{softmax}(\*\ell_t^i)$, where $\*\pi_t^i=\*\pi_t\*e_i$ denotes the distribution for the $i$-th position. We follow the literature and denote the top-1 probability at each position as its \emph{confidence} \citep{cheng2025SDAR,nie2025LLaDA}.
After each forward pass, the model reveals the masked positions it is most confident about, replacing their \texttt{[MASK]} tokens with the predicted tokens to form $\*x_{t-1}$. This repeats until all positions are revealed.

\textbf{Block Diffusion.} In the block diffusion variant~\citep{arriola2025BD3LM}, generation is partitioned into blocks of tokens produced left to right. Previously generated blocks serve as fully unmasked context while the current block is denoised through the loop described above. Our method applies identically within each block; all results in this paper use the block diffusion setting.

\section{Multi-Token Residual Prediction}
\label{sec:method}

\subsection{Motivation}

Our motivation starts with a natural question following Section~\ref{sec:prelim}: \emph{how close are the logits $\*\ell_t$ and $\*\ell_{t-1}$ in practice, and can we characterize the structure of their difference?} We address this question empirically and then provide theoretical grounding.

\textbf{Empirical characterization of the residual.} We define the per-step logit residual as
\begin{equation}
\*\delta_t = \*\ell_{t-1} - \*\ell_t \in \mathbb{R}^{L \times V},
\end{equation}
where $\*\ell_t$ and $\*\ell_{t-1}$ are the backbone logits before and after unmasking a set of positions. Analogously, we define the per-step hidden-state residual $\*\Delta_t = \*h_{t-1} - \*h_t$ between the backbone hidden states at adjacent denoising steps. To probe the residual at different effective step sizes, we run greedy block-diffusion decoding on the GSM8K evaluation set with SDAR-1.7B/4B/8B, and compare backbone states that are $r$ revealed-token steps apart, with revealed token count $r$ ranging from $1$ to $16$ (a full block). Figure~\ref{fig:motivation} reports the resulting residuals; see Appendix~\ref{app:residual_magnitude} for the full measurement protocol.

\begin{figure}[t!]
    \centering
    \includegraphics[width=0.5\linewidth]{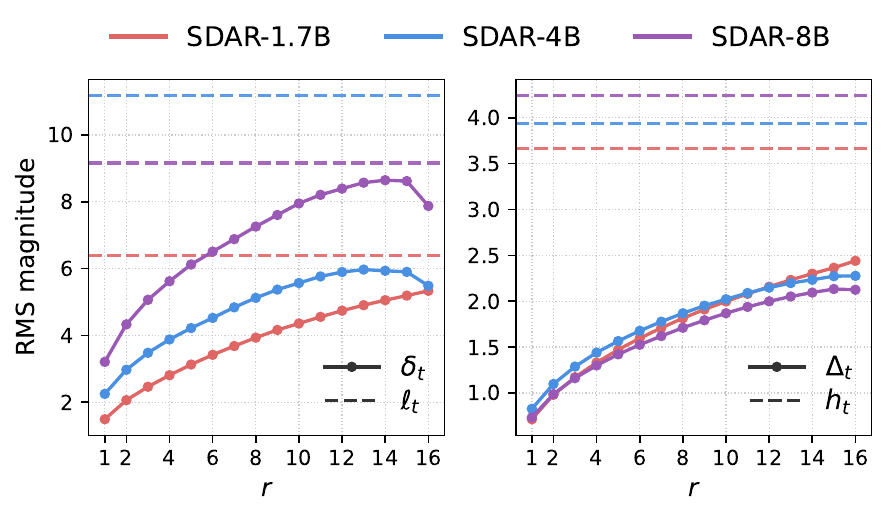}
    \caption{Magnitudes of the per-step residual versus the full backbone state,
    measured on GSM8K and averaged per entry (RMS). The horizontal axis $r$ is
    the number of tokens unmasked between two adjacent denoising states $\*x_t$
    and $\*x_{t-1}$; line color indicates the SDAR backbone scale (1.7B / 4B /
    8B). \textbf{Left:} logits. Solid curves with markers show the residual
    $\*\delta_t = \*\ell_{t-1} - \*\ell_t$; dashed lines show the average
    magnitude of the backbone logits $\*\ell_t$ for reference. \textbf{Right:}
    hidden states. Solid curves show the residual $\*\Delta_t = \*h_{t-1} -
    \*h_t$; dashed lines show the average magnitude of the backbone hidden
    states $\*h_t$. Across all three scales and both representations, the
    residual is consistently smaller than the underlying state, with the gap
    widening sharply for small $r$ (the practical decoding regime).}
    \label{fig:motivation}
\end{figure}
The key observation here is the \emph{magnitude gap}: the RMS magnitude of $\*\delta_t$ is consistently and significantly smaller than that of the full logits $\*\ell_t$. When $r=1$, the residual is around 20\% of the magnitude of the logits; even at $r=16$ (unmasking an entire block at once), the residual remains below the full logits in magnitude across all model scales. This gap widens for smaller $r$, which is precisely the operating regime of practical denoising. This suggests that predicting $\*\delta_t$ requires capturing a signal with substantially lower dynamic range than predicting $\*\ell_{t-1}$ from scratch.

\textbf{Theoretical justification.} The empirical observations above can be grounded in the Markov chain structure of the denoising process. The sequence of partially masked states $\*x_T, \*x_{T-1}, \ldots, \*x_0$ forms a Markov chain: each state $\*x_{t-1}$ is obtained from $\*x_t$ by unmasking a subset of positions $\mathcal{R}_t$, so the transition depends only on the current state. The backbone output $\*\ell_t = f(\*x_t)$ is a deterministic function of the chain's state. We exploit this structure to bound the residual $\*\delta_t = \*\ell_{t-1} - \*\ell_t$.

The key observation is that the transition $\*x_t \to \*x_{t-1}$ is a \emph{small perturbation}: it modifies only $|\mathcal{R}_t|$ out of $L$ positions, leaving the remaining $L - |\mathcal{R}_t|$ positions unchanged. For any masked position $i \notin \mathcal{R}_t$, the change in its predictive distribution is driven entirely by the indirect influence of the newly revealed tokens through the backbone's attention mechanism. Recall from Section~\ref{sec:prelim} the predictive distribution $\*\pi_t^i = \mathrm{softmax}(\*\ell_t^i)$. We bound the change across one transition as follows.

\begin{proposition}[One-step contraction]
\label{prop:contraction}
If the backbone $f$ is $\kappa$-Lipschitz with respect to its input embeddings, then for any masked position $i \notin \mathcal{R}_t$:
\begin{equation}
    D_{\mathrm{TV}}(\*\pi_{t-1}^i,\, \*\pi_t^i) \leq \kappa \cdot \frac{|\mathcal{R}_t|}{L} \cdot \max_{j \in \mathcal{R}_t} \|\*e_{v_j} - \*e_{\texttt{[MASK]}}\|_2,
\end{equation}
where $\*e_v$ denotes the embedding of token $v$, $\*e_{\texttt{[MASK]}}$ denotes the embedding of mask token and $D_{\mathrm{TV}}$ is the total variation distance.
\end{proposition}

The bound is controlled by three intuitive quantities: the model's sensitivity $\kappa$, the fraction of positions perturbed $|\mathcal{R}_t|/L$, and the embedding distance between the revealed tokens and the mask token. Since $|\mathcal{R}_t| \ll L$ at each step\footnote{In some recent models such as SDAR, $L$ could be small for a single block (e.g., 16). However, this doesn't violate our bound or assumption.}, the perturbation ratio is small and the TV distance is correspondingly bounded.

Applying this bound recursively along the chain yields a cumulative characterization:

\begin{corollary}[Cumulative residual decay]
\label{cor:decay}
As denoising progresses, the number of remaining masked positions $m_t = |\{i : x_t^i = \texttt{[MASK]}\}|$ decreases monotonically. At later steps, the perturbation ratio $|\mathcal{R}_t|/L$ remains small while the predictions $\*\pi_t^i$ concentrate (the backbone becomes more confident on fewer remaining positions). Consequently, $\|\*\delta_t\|$ decreases along the chain: later denoising steps produce smaller residuals than earlier ones.
\end{corollary}

We provide the full proofs in Appendix~\ref{app:proofs}. These results justify the design of MRP: the residual $\*\delta_t$ is a low-complexity prediction target, learnable by a module far smaller than the backbone.

\subsection{Architecture}
\label{sec:arch}

The MRP module follows an architectural pattern similar to EAGLE~\citep{li2024eagle} and Multi-Token Prediction (MTP)~\citep{liu2024DeepSeekV3}: a lightweight head that receives hidden states from a frozen backbone and produces additional token predictions. The key difference is that MRP outputs a \emph{residual correction} rather than the entire distribution. The final logits are obtained by summing the backbone logits and the MRP output, $\*\ell_t + \*\delta_t$, so the module only needs to learn the deviation from the backbone's prediction rather than a complete distribution over the vocabulary.
The MRP module takes the backbone's hidden states $\*h$ and the token embeddings of the revealed sequence as input, fuses them via a learned projection, passes the result through a small number of transformer layers, and reuses the backbone's language model head to produce the residual logits $\*\delta_t$. The corrected logits are $\*\ell_t + \*\delta_t$. The full architecture is illustrated in Figure~\ref{fig:1}.

\subsection{Training}
\label{sec:training}

The MRP module is trained to predict the log-density residual between two consecutive denoising states. Given a clean sequence $\*x_0$, we first corrupt it by independently masking response tokens with probability $\mathrm{Uniform}(0,1)$ to obtain the noisy input $\*x_t$. We then reveal $r$ ground-truth tokens per block to form the partially revealed input $\*x_{t-1}$, simulating one denoising step. The frozen backbone $f$ is run twice: once on $\*x_t$ to produce hidden states $\*h$ and log-densities $\log p(\*x_0|\*x_t)$, and once on $\*x_{t-1}$ to produce the teacher target $\log p(\*x_0|\*x_{t-1})$ without gradients. The MRP module $g$ takes $\*x_{t-1}$ and the backbone hidden states $\*h$ as input and predicts the residual $\*\delta_t$. Training encourages $\*\delta_t$ to match the teacher output $\log p(\*x_0|\*x_{t-1}) - \log p(\*x_0|\*x_t)$ via a KL divergence objective restricted to still-masked positions:
\begin{equation}
    \mathcal{L} = \mathrm{KL}\!\Big(\, \mathrm{softmax}\!\big(\*\ell_{t-1}\big) \;\Big\|\; \mathrm{softmax}\!\big(\*\ell_t + \*\delta_t\big) \,\Big)\Big|_{\text{masked}},
\end{equation}
where $\*\ell_t = f(\*x_t)$ and $\*\ell_{t-1} = f(\*x_{t-1})$ are the backbone logits on the noisy and revealed inputs respectively. Since the logits are unnormalized log-densities, and the KL divergence is computed after softmax, the normalizing constants cancel and the objective is equivalent to matching the true conditional distributions.
The only trainable parameters are those of the MRP module $g$. Algorithm~\ref{alg:training} summarizes the full procedure.

\begin{algorithm}[h]
\caption{MRP Training}
\label{alg:training}
\begin{algorithmic}[1]
\Require Backbone $f$ (frozen), MRP module $g$ (trainable)
\For{each training sequence $\*x_0$}
    \State Mask response tokens independently with probability $\mathrm{Uniform}(0,1)$ to obtain $\*x_t$
    \State $\*h,\, \*\ell_t \gets f(\*x_t)$ \Comment{backbone forward: hidden states + logits}
    \State Reveal $r$ ground-truth tokens per block to obtain $\*x_{t-1}$
    \State $\*\delta_t \gets g(\*x_{t},\, \*h)$ \Comment{MRP residual prediction}
    \State $\*\ell_{t-1} \gets f(\*x_{t-1})$ \Comment{Obtain teacher target}
    \State $\mathcal{L} \gets \mathrm{KL}\!\Big(\, \mathrm{softmax}\!\big(\*\ell_{t-1}\big) \;\Big\|\; \mathrm{softmax}\!\big(\*\ell_t + \*\delta_t\big) \,\Big)\Big|_{\text{masked}}$
    \State Update $g$ with $\nabla \mathcal{L}$
\EndFor
\end{algorithmic}
\end{algorithm}

\subsection{Inference}
\label{sec:inference}

At inference time, the MRP module augments each backbone forward pass with additional residual prediction, allowing more tokens to be resolved per backbone forward. After the backbone processes $\*x_t$ and unmasks high-confidence positions to form $\*x_{t-1}$, the MRP module predicts the residual $\*\delta_t$ from the backbone's hidden states and the updated sequence, and the corrected logits $\*\ell_t + \*\delta_t$ approximate a second backbone forward at a fraction of the cost. How this correction is used depends on the decoding regime. In the \emph{static} regime, where each step unmasks a fixed (typically small) number of tokens and quality is the priority, MRP \emph{adds} reveals beyond the threshold through two modes: \emph{direct} decoding, which uses the corrected logits without verification for a tunable quality--throughput tradeoff, and \emph{speculative} decoding, which verifies them against the backbone for lossless acceleration. In the \emph{dynamic} regime, where a confidence threshold unmasks many tokens per step and accuracy degrades, MRP instead \emph{revokes} over-eager reveals through \emph{remasking} decoding. We describe each in turn.

\begin{algorithm}[H]
\caption{MRP Inference (Direct Decoding Mode) in Static Denoising}
\label{alg:inference_direct}
\begin{algorithmic}[1]
\Require backbone $f$, MRP module $g$, MRP steps $K$, number of reveal tokens $r$, fully masked sequence $\*x_T$
\State $t \gets T$
\While{$\*x_t$ contains \texttt{[MASK]} tokens}
    \State $\*h,\, \*\ell \gets f(\*x_t)$ \Comment{backbone forward to produce hidden states and logits}
    \State Reveal $r$ tokens with highest confidence in $\*x_{t}$
    \For{$k = 1, \ldots, K$}
        \State $\*\Delta,\, \*\delta \gets g(\*x_{t},\, \*h)$ \Comment{MRP step to produce hidden residual and logits residual}
        \State $\*h \gets \*h + \*\Delta$;\quad $\*\ell \gets \*\ell + \*\delta$
        \State Reveal $r$ tokens with highest confidence in $\*x_{t}$
    \EndFor
    \State $t \gets t - 1$
\EndWhile
\State \Return $\*x_t$
\end{algorithmic}
\end{algorithm}

\textbf{Direct decoding (Algorithm~\ref{alg:inference_direct}).} The simplest use of MRP in the static regime is to unmask additional positions directly from the corrected logits, without any verification. After the backbone unmasks high-confidence positions, the MRP module runs $K$ additional decoding steps. At each step, the module predicts both a hidden state residual $\*\Delta$ and a logit residual $\*\delta$, which are accumulated into the running hidden states and logits: $\*h \gets \*h + \*\Delta$, $\*\ell \gets \*\ell + \*\delta$. $r$ tokens with highest confidence are unmasked immediately. This introduces a small approximation error. Larger $K$ unmasks more tokens per backbone forward at the cost of compounding approximation error; we study this tradeoff in Section~\ref{sec:experiments}. Direct decoding thus sits on a strictly more favorable quality--throughput frontier than threshold-based decoding at matched throughput, at the price of a controlled quality loss.

\begin{algorithm}[H]
\caption{MRP Inference (Speculative Decoding Mode) in Static Denoising}
\label{alg:inference_spec}
\begin{algorithmic}[1]
\Require backbone $f$, MRP module $g$, MRP steps $K$, number of reveal tokens $r$, fully masked sequence $\*x_T$
\State $t \gets T$
\While{$\*x_t$ contains \texttt{[MASK]} tokens}
    \State $\*h,\, \*\ell \gets f(\*x_t)$ \Comment{backbone forward}
    \State Reveal $r$ tokens with highest confidence in $\*x_{t}$
    \For{$k = 1, \ldots, K$} \Comment{MRP drafting}
        \State $\*\Delta,\, \*\delta \gets g(\*x_{t},\, \*h)$
        \State $\*h \gets \*h + \*\Delta$;\quad $\*\ell \gets \*\ell + \*\delta$
        \State Draft token with highest confidence into $\*x_{t-1}$
    \EndFor
    \State $\*h',\, \*\ell' \gets f(\*x_{t-1})$ \Comment{backbone verify}
    \State Reject drafted positions where $\arg\max \,\ell'^{i} \neq \arg\max \,\ell^{i}$; remask rejected positions in $\*x_{t-1}$
    \State $t \gets t - 1$
\EndWhile
\State \Return $\*x_t$
\end{algorithmic}
\end{algorithm}

\textbf{Speculative decoding (Algorithm~\ref{alg:inference_spec}).} When output quality must be preserved, the same MRP-decoded tokens are used as a draft that the backbone verifies, as in speculative decoding for autoregressive models \citep{liu2024DeepSeekV3}. The drafting steps are identical to direct decoding. But instead of committing the drafts, we adopt the verification framework of Self Speculative Decoding (SSD)~\citep{gao2025SSD}: after all $K$ MRP steps complete, the backbone runs a single verification forward pass on the updated sequence $\*x_{t-1}$, producing $\*\ell' = f(\*x_{t-1})$. Drafted positions where the backbone's prediction matches the draft ($\arg\max\, \ell'^{i} = \arg\max\, \ell^{i}$) are accepted; positions where it disagrees are remasked. Since accepted tokens are exactly those the backbone itself would have produced, acceptance incurs no quality loss, making this mode entirely \emph{lossless}. This position-wise verification is natural for diffusion models, where all positions are predicted simultaneously in a single forward pass. The verification pass is not wasted: its hidden states and logits seed the next iteration, so when the acceptance rate is high, the verification forward effectively serves as the backbone forward for the next step, amortizing its cost.

\begin{algorithm}[H]
\caption{MRP Inference (Remasking Mode) in Dynamic Denoising}
\label{alg:inference_remask}
\begin{algorithmic}[1]
\Require backbone $f$, MRP module $g$, threshold $\tau$, fully masked sequence $\*x_T$
\State $t \gets T$
\While{$\*x_t$ contains \texttt{[MASK]} tokens}
    \State $\*h,\, \*\ell \gets f(\*x_t)$ \Comment{backbone forward}
    \State Reveal all the tokens with confidence exceeding a threshold $\tau$, forming a set $\mathcal{R}$
    \State $\*\Delta,\, \*\delta \gets g(\*x_{t},\, \*h)$
    \State $\*h \gets \*h + \*\Delta$;\quad $\*\ell \gets \*\ell + \*\delta$ \Comment{MRP forward}
    \State Remask positions $i \in \mathcal{R}$ where confidence is below $\tau$ \Comment{remask tokens below the same $\tau$}
    \State $t \gets t - 1$
\EndWhile
\State \Return $\*x_t$
\end{algorithmic}
\end{algorithm}

\textbf{Remasking decoding (Algorithm~\ref{alg:inference_remask}).} The two modes above operate in the static regime and use MRP to \emph{add} reveals beyond what the backbone's threshold yields. In the dynamic, high-throughput regime, a low confidence threshold causes the backbone to unmask many positions simultaneously, and accuracy degrades sharply as a result. 
Here MRP instead \emph{remasks} over-eager reveals. At a low threshold $\tau$, the backbone reveals the set $\mathcal{R}$ of positions whose confidence exceeds $\tau$. A single MRP pass then predicts the residual $\*\delta$ conditioned on these fresh reveals, giving the estimated logits for the subsequent step $\*\ell + \*\delta$. 
We then use the estimated next-step logits to remask revealed tokens in the current step.
The intuition is that a single-step confidence estimate is noisy: a position clearing $\tau$ once may do so spuriously, whereas a position that stays above $\tau$ across two consecutive steps has more stably crossed the threshold, and only such positions are kept committed. This double-check echoes the re-masking principle of WINO~\citep{hong2025WINO}, but replaces its second backbone pass with a single lightweight MRP forward. Using the same $\tau$ for both the reveal and the remasking, it trades one extra MRP forward for higher accuracy at aggressive thresholds, and introduces no tuning beyond what dynamic decoding already requires (Section~\ref{sec:exp:dynamic}).

\textbf{Complementary Operating regimes.} The decoding strategies map onto the two operating regimes of DLM inference rather than competing within one. In the static regime, where few tokens are unmasked per step and quality is paramount, MRP adds reveals: direct decoding trades a controlled quality loss for throughput, while speculative decoding preserves backbone quality through verification. In the dynamic regime, where a low threshold unmasks aggressively for speed, MRP revokes reveals through remasking, recovering most of the accuracy that aggressive unmasking would otherwise sacrifice. All strategies share the same trained MRP module; the choice is made at serving time with no retraining.

\section{Experiments}
\label{sec:experiments}

In this section, we present our main empirical results. We first validate the residual learning hypothesis by showing that predicting the inter-step residual is substantially easier than predicting the full next-step distribution (Section~\ref{sec:exp:residual}). We then evaluate MRP under both budget-based static denoising (Section~\ref{sec:exp:static}) and threshold-based dynamic denoising (Section~\ref{sec:exp:dynamic}), and investigate how the MRP module scales with MRP depth (Section~\ref{sec:exp:scaling}).

\subsection{Setup}
\label{sec:training-setup}

\textbf{Models.} We attach our MRP module to the
SDAR-Chat family of block-diffusion language models~\citep{cheng2025SDAR} at
three scales (1.7B, 4B, 8B), all trained with block size 16. The MRP module is
a $D{=}3$-layer transformer that follows the architecture of
Section~\ref{sec:arch}; the backbone, language-modeling head, and token
embeddings remain \emph{frozen} during training.

\textbf{Dataset.} We adopt the training dataset SFTDatasetV3~\citep{wang2024MambaInLlama},
a 12.4M-conversation supervised fine-tuning corpus assembled from three public instruction-tuning sources:
OpenHermes-2.5~\citep{teknium2024OpenHermes}, GenQA~\citep{chen2024GenQA},
and Infinity-Instruct~\citep{baai2024InfinityInstruct}, and reformatted into a uniform
two- to eight-turn user/assistant message schema. The mixture spans diverse
domains (reasoning, mathematics, code, dialogue) at a scale large enough that
the MRP module never revisits a sample, which is well-suited to the
distillation-style residual objective we train against.

\textbf{Training.} We train with the residual KD loss of
Algorithm~\ref{alg:training} at temperature $T_{\mathrm{KD}}{=}1.0$ in the forward direction
$\mathrm{KL}$, restricted to
still-masked positions, and unroll $K{=}2$ MRP iterations per backward with
uniform per-step weighting. We use AdamW with peak learning rate
$1\mathrm{e}{-3}$ on a cosine schedule for a single epoch. The full
hyperparameter list is given in Appendix~\ref{app:hyperparams}
(Table~\ref{tab:hyperparams}).

\textbf{Evaluation.}
We follow the evaluation protocol of SDAR~\citep{cheng2025SDAR} and focus on the two generation task categories:
\emph{(i) Mathematics} -- GSM8K (0-shot, CoT)~\citep{cobbe2021GSM8K} and
MATH500 (0-shot, CoT)~\citep{hendrycks2021MATH,lightman2024PRM800K};
and \emph{(ii) Code Generation} -- HumanEval~\citep{chen2021HumanEval} and
MBPP~\citep{austin2021MBPP}, both evaluated zero-shot.
Following SDAR, we use greedy decoding for all
models (temperature $1.0$, Top-K${=}1$, Top-P${=}1.0$) with a maximum
generation length of $4096$ tokens. The SDAR-Chat backbones we build on are
the block-size-16 variants, so block length and number of denoising steps per
block are both set to $16$.

\textbf{System.}
We implement MRP in the SGLang inference engine \citep{zheng2024sglang}.
Unless otherwise stated, all the experimental numbers are obtained on a single NVIDIA H100 GPU.

\begin{table}[t!]
    \centering
    \tabfont
    \begin{tabular}{l c rrrr rrrr rrrr}
        \toprule
        & & \multicolumn{4}{c}{\textbf{1.7B}} & \multicolumn{4}{c}{\textbf{4B}} & \multicolumn{4}{c}{\textbf{8B}} \\
        \cmidrule(lr){3-6} \cmidrule(lr){7-10} \cmidrule(lr){11-14}
        \textbf{Type} & & $K{=}1$ & $K{=}2$ & $K{=}3$ & $K{=}4$ & $K{=}1$ & $K{=}2$ & $K{=}3$ & $K{=}4$ & $K{=}1$ & $K{=}2$ & $K{=}3$ & $K{=}4$ \\
        \midrule
        Direct Modeling                 & & $76.5$ & $73.1$ & $62.4$ & $48.0$ & $84.8$ & $19.6$ & $5.9$ & $1.9$ & $90.5$ & $72.0$ & $34.2$ & $6.7$  \\
        Residual Modeling               & & $76.7$ & $73.9$ & $66.7$ & $60.9$ & $88.6$ & $87.6$ & $70.9$ & $57.2$ & $90.1$ & $89.2$ & $84.8$ & $79.8$  \\
        $\Delta$ (Res. $-$ Dir.) & & \textcolor{betterGreen}{+$0.2$} & \textcolor{betterGreen}{+$0.8$} & \textcolor{betterGreen}{+$4.3$} & \textcolor{betterGreen}{+$12.9$} & \textcolor{betterGreen}{+$3.8$} & \textcolor{betterGreen}{+$68.0$} & \textcolor{betterGreen}{+$65.0$} & \textcolor{betterGreen}{+$55.3$} & \textcolor{worseRed}{-$0.4$} & \textcolor{betterGreen}{+$17.2$} & \textcolor{betterGreen}{+$50.6$} & \textcolor{betterGreen}{+$73.1$} \\
        \bottomrule
    \end{tabular}
    \vspace{0.9em}
    \caption{GSM8K (0-shot, CoT) accuracy (\%) of MRP trained with the
    \emph{residual} objective versus a \emph{direct} distillation
    objective, evaluated at $K{=}\{1,2,3,4\}$ MRP steps per backbone forward. Both
    variants share identical architecture and hyperparameters and are
    trained on the same data
    (\texttt{sftdatasetv3}~\citep{wang2024MambaInLlama}); they are
    evaluated under static low-confidence decoding. The $\Delta$ row
    reports the improvement of residual over direct training, with
    \textcolor{betterGreen}{green} indicating residual is better and
    \textcolor{worseRed}{red} indicating residual is worse.}
    \label{tab:residual_ablation}
\end{table}

\begin{table}[t!]
    \centering
    \tabfont
    \begin{tabular}{ll cc cc cc cc}
        \toprule
        \multirow{2}{*}[-0.5ex]{\textbf{Model}} & \multirow{2}{*}[-0.5ex]{\textbf{Setting}}
            & \multicolumn{2}{c}{\textbf{GSM8K}}
            & \multicolumn{2}{c}{\textbf{MATH500}}
            & \multicolumn{2}{c}{\textbf{HumanEval}}
            & \multicolumn{2}{c}{\textbf{MBPP}} \\
        \cmidrule(lr){3-4} \cmidrule(lr){5-6} \cmidrule(lr){7-8} \cmidrule(lr){9-10}
            & & Acc. & TPS & Acc. & TPS & Acc. & TPS & Acc. & TPS \\
        \midrule
        \noalign{\vskip-\belowrulesep}
        \multirow{5}{*}{1.7B}
            & \cellcolor{blue!8}Base ($r{=}1$)          & \cellcolor{blue!8}$77.6$ & \cellcolor{blue!8}$339.0$ & \cellcolor{blue!8}$55.0$ & \cellcolor{blue!8}$358.0$ & \cellcolor{blue!8}$53.7$ & \cellcolor{blue!8}$314.0$ & \cellcolor{blue!8}$52.9$ & \cellcolor{blue!8}$282.0$ \\
            & Base ($r{=}2$)          & $70.8$ & $1.79\times$ & $45.8$ & $1.79\times$ & $29.9$ & $1.70\times$ & $32.3$ & $1.92\times$ \\
            & Direct ($K{=}1$)        & $76.7$ & $1.50\times$ & $54.2$ & $1.45\times$ & $48.2$ & $1.45\times$ & $45.5$ & $1.45\times$ \\
            & Direct ($K{=}2$)        & $73.9$ & $1.71\times$ & $50.6$ & $1.65\times$ & $41.5$ & $1.53\times$ & $42.4$ & $1.72\times$ \\
            & Spec ($K{=}3$)      & $78.3$ & $1.25\times$ & $58.4$ & $1.23\times$ & $55.5$ & $1.23\times$ & $53.3$ & $1.17\times$ \\
        \midrule
        \noalign{\vskip-\belowrulesep}
        \multirow{5}{*}{4B}
            & \cellcolor{blue!8}Base ($r{=}1$)          & \cellcolor{blue!8}$90.3$ & \cellcolor{blue!8}$188.0$ & \cellcolor{blue!8}$71.4$ & \cellcolor{blue!8}$204.0$ & \cellcolor{blue!8}$67.1$ & \cellcolor{blue!8}$170.0$ & \cellcolor{blue!8}$66.5$ & \cellcolor{blue!8}$161.0$ \\
            & Base ($r{=}2$)          & $86.4$ & $1.84\times$ & $64.0$ & $1.81\times$ & $45.7$ & $1.74\times$ & $37.7$ & $1.81\times$ \\
            & Direct ($K{=}1$)        & $88.6$ & $1.56\times$ & $64.4$ & $1.54\times$ & $62.8$ & $1.55\times$ & $55.6$ & $1.43\times$ \\
            & Direct ($K{=}2$)        & $87.6$ & $1.84\times$ & $60.4$ & $1.80\times$ & $53.0$ & $1.79\times$ & $51.0$ & $1.82\times$ \\
            & Spec ($K{=}3$)      & $90.0$ & $1.36\times$ & $68.0$ & $1.26\times$ & $67.7$ & $1.35\times$ & $66.5$ & $1.27\times$ \\
        \midrule
        \noalign{\vskip-\belowrulesep}
        \multirow{5}{*}{8B}
            & \cellcolor{blue!8}Base ($r{=}1$)          & \cellcolor{blue!8}$90.9$ & \cellcolor{blue!8}$126.0$ & \cellcolor{blue!8}$72.2$ & \cellcolor{blue!8}$132.0$ & \cellcolor{blue!8}$73.8$ & \cellcolor{blue!8}$114.0$ & \cellcolor{blue!8}$67.7$ & \cellcolor{blue!8}$102.0$ \\
            & Base ($r{=}2$)          & $87.7$ & $1.80\times$ & $68.4$ & $1.81\times$ & $42.1$ & $1.74\times$ & $48.2$ & $1.60\times$ \\
            & Direct ($K{=}1$)        & $90.1$ & $1.59\times$ & $71.4$ & $1.61\times$ & $67.1$ & $1.53\times$ & $63.8$ & $1.51\times$ \\
            & Direct ($K{=}2$)        & $89.2$ & $1.89\times$ & $70.8$ & $1.91\times$ & $64.0$ & $1.78\times$ & $59.9$ & $1.75\times$ \\
            & Spec ($K{=}3$)      & $90.4$ & $1.40\times$ & $74.8$ & $1.39\times$ & $72.6$ & $1.34\times$ & $67.3$ & $1.34\times$ \\
        \bottomrule
    \end{tabular}
    \vspace{0.5em}
    \caption{MRP under static denoising across model scales. \emph{Base ($r$)} denotes the
    backbone-only baseline that unmasks $r$ tokens per denoising step.
    \emph{Direct} mode runs MRP with $K{=}\{1,2\}$ steps; \emph{speculative} mode runs
    MRP with $K{=}3$ steps. }
    \label{tab:main}
\end{table}

\textbf{Denoising Mode.}
We evaluate MRP under both denoising modes described in SDAR~\citep{cheng2025SDAR}. In \emph{static denoising}, each denoising step unmasks a fixed number of $r$ positions per block by selecting the highest-confidence predictions. In \emph{dynamic denoising}, all positions whose confidence exceeds a threshold $\tau$ are unmasked simultaneously, with a minimum of one position revealed per step to guarantee progress.

\subsection{Residual Modeling}
\label{sec:exp:residual}
We start by comparing two modeling variants trained with identical protocols: (i) \emph{Residual Modeling}, which predicts $\*\delta_t = \log p(\*x_0|\*x_{t-1}) - \log p(\*x_0|\*x_t)$ on top of the frozen backbone logits, and (ii) \emph{Direct Modeling}, which is distilled to predict $\log p(\*x_0|\*x_{t-1})$ directly without the residual reparameterization. Table~\ref{tab:residual_ablation} reports GSM8K accuracy under direct decoding ($\tau{=}1.0$) at $K\in\{1,2,3,4\}$ MRP steps across all three SDAR scales. The two objectives are close at $K{=}1$ (gaps of $+0.2$, $+3.8$, and $-0.4$ on 1.7B, 4B, and 8B), but the gap opens sharply as $K$ increases: at $K{=}2$ residual training already leads by $+68.0$ on 4B and $+17.2$ on 8B, and by $K{=}4$ the direct model has collapsed at every scale ($\Delta$ of $+12.9$, $+55.3$, and $+73.1$). The reason is that direct distillation must reproduce the full log-density at every iteration, so per-step errors compound across steps, whereas residual training only models a low-magnitude correction and extrapolates more reliably to deeper MRP steps.
This substantiates our motivation for residual modeling.

\subsection{MRP Inference under Static Denoising}
\label{sec:exp:static}

Table~\ref{tab:main} reports the static denoising results. For each task we compare the unmodified
backbone (\emph{Base}, $r{=}1$ token unmasked per step) against our
\emph{Direct} mode at $K{\in}\{1,2\}$ MRP steps and our \emph{Spec}
mode at $K{=}3$, and report task accuracy together with end-to-end
throughput.

\textbf{Speculative mode preserves quality at meaningful speedup.} In every setting, the speculative mode has no quality drop as expected. It delivers
$1.17\text{--}1.40\times$ speed up in throughput across all configurations, with the largest absolute speedups on the larger backbone where
the verification cost is best amortized (i.e., $1.40\times$ on 8B/GSM8K).

\textbf{Direct mode buys further throughput at a controlled quality
cost.} At $K{=}1$, the direct mode has close quality compared to backbone
(e.g., 88.6 vs.\ 90.3 GSM8K at 4B; 90.1 vs.\ 90.9 at 8B) while running
$1.43\text{--}1.61\times$ faster. Pushing to $K{=}2$ adds another step
of speedup ($1.53\text{--}1.91\times$) with a modest accuracy hit
concentrated on code benchmarks (e.g., HumanEval). The clean
monotonic tradeoff between $K$ and accuracy lets practitioners pick the
operating point: \emph{Speculative Mode} when quality is non-negotiable, \emph{Direct Mode}
($K{=}1$ or $2$) when latency dominates.

\subsection{MRP Inference under Dynamic Denoising}
\label{sec:exp:dynamic}

Table~\ref{tab:remask-acc} evaluates the remasking mode (Algorithm~\ref{alg:inference_remask}) under threshold-based dynamic denoising. We sweep the unmasking threshold $\tau \in \{0.5, 0.6, 0.7, 0.8, 0.9\}$ on SDAR-1.7B/4B/8B. For each $\tau$ we compare plain threshold-based \emph{Dynamic} decoding against \emph{MRP Remask}, reporting task accuracy on GSM8K, MATH500, HumanEval, and MBPP together with the accuracy change over Dynamic at the matched $\tau$. As described in Section~\ref{sec:inference}, the reveal and the remasking share the same $\tau$, so it introduces no tuning beyond what dynamic decoding already requires. The intuition is that a single forward pass gives a noisy confidence estimate: a position may clear $\tau$ once by chance, but a position that stays above $\tau$ across two consecutive estimates has more stably crossed the threshold. MRP Remask keeps only these doubly-confirmed positions and revokes the rest, filtering out the premature reveals that a single noisy estimate would commit. We deliberately study low thresholds: these make the backbone unmask many positions per step, which is precisely the regime where simultaneous reveals are most likely to be premature and where MRP remasking has the most to correct.

In addition, in our experiments we adopt the \emph{max-accept} trick from WINO~\citep{hong2025WINO}, which caps the number of positions unmasked per step at $m = \min(\max(\lfloor 0.7\,n\rfloor, 5),\, L)$, where $n$ is the number of still-masked positions in the block and $L$ denotes the block size; when more than $m$ positions clear $\tau$, we keep the $m$ highest-confidence ones.

\begin{table}[H]
    \centering
    \tabfont
    \begin{tabular}{cl ll ll ll ll}
        \toprule
        \multirow{2}{*}[-0.5ex]{\textbf{Model}} & \multirow{2}{*}[-0.5ex]{\textbf{$\tau$}}
            & \multicolumn{2}{c}{\textbf{GSM8K}}
            & \multicolumn{2}{c}{\textbf{MATH500}}
            & \multicolumn{2}{c}{\textbf{HumanEval}}
            & \multicolumn{2}{c}{\textbf{MBPP}} \\
        \cmidrule(lr){3-4} \cmidrule(lr){5-6} \cmidrule(lr){7-8} \cmidrule(lr){9-10}
            & & Dyn. & Remask & Dyn. & Remask & Dyn. & Remask & Dyn. & Remask \\
        \midrule
        \multirow{5}{*}{1.7B}
            & $0.5$ & $41.6$ & $59.1$\dgain{17.5} & $26.0$ & $37.4$\dgain{11.4} & $17.7$ & $28.7$\dgain{11.0} & $26.9$ & $41.3$\dgain{14.4} \\
            & $0.6$ & $56.3$ & $67.0$\dgain{10.7} & $33.4$ & $40.4$\dgain{7.0} & $31.7$ & $43.3$\dgain{11.6} & $42.4$ & $49.0$\dgain{6.6} \\
            & $0.7$ & $65.4$ & $71.8$\dgain{6.4} & $39.4$ & $48.6$\dgain{9.2} & $40.9$ & $45.1$\dgain{4.2} & $49.8$ & $51.0$\dgain{1.2} \\
            & $0.8$ & $70.6$ & $75.4$\dgain{4.8} & $47.4$ & $52.0$\dgain{4.6} & $45.7$ & $48.8$\dgain{3.1} & $51.8$ & $51.8$\dnil \\
            & $0.9$ & $76.2$ & $77.3$\dgain{1.1} & $51.2$ & $57.0$\dgain{5.8} & $49.4$ & $52.4$\dgain{3.0} & $53.7$ & $54.1$\dgain{0.4} \\
        \midrule
        \multirow{5}{*}{4B}
            & $0.5$ & $63.4$ & $81.1$\dgain{17.7} & $44.2$ & $58.4$\dgain{14.2} & $32.3$ & $53.1$\dgain{20.8} & $38.5$ & $50.6$\dgain{12.1} \\
            & $0.6$ & $76.4$ & $85.5$\dgain{9.1} & $53.6$ & $61.4$\dgain{7.8} & $49.4$ & $57.9$\dgain{8.5} & $49.8$ & $57.2$\dgain{7.4} \\
            & $0.7$ & $84.6$ & $88.5$\dgain{3.9} & $60.4$ & $65.6$\dgain{5.2} & $60.4$ & $62.2$\dgain{1.8} & $61.1$ & $63.4$\dgain{2.3} \\
            & $0.8$ & $87.9$ & $90.1$\dgain{2.2} & $66.8$ & $70.6$\dgain{3.8} & $64.6$ & $62.8$\dloss{-1.8} & $63.8$ & $64.2$\dgain{0.4} \\
            & $0.9$ & $88.5$ & $90.1$\dgain{1.6} & $69.0$ & $70.6$\dgain{1.6} & $67.1$ & $65.9$\dloss{-1.2} & $65.4$ & $64.6$\dloss{-0.8} \\
        \midrule
        \multirow{5}{*}{8B}
            & $0.5$ & $67.9$ & $82.3$\dgain{14.4} & $45.2$ & $58.0$\dgain{12.8} & $32.3$ & $54.9$\dgain{22.6} & $34.6$ & $49.4$\dgain{14.8} \\
            & $0.6$ & $79.6$ & $86.8$\dgain{7.2} & $54.8$ & $63.8$\dgain{9.0} & $48.8$ & $63.4$\dgain{14.6} & $48.3$ & $59.9$\dgain{11.6} \\
            & $0.7$ & $85.9$ & $89.0$\dgain{3.1} & $60.8$ & $69.0$\dgain{8.2} & $64.6$ & $72.6$\dgain{8.0} & $54.9$ & $60.3$\dgain{5.4} \\
            & $0.8$ & $89.3$ & $91.0$\dgain{1.7} & $68.0$ & $70.2$\dgain{2.2} & $74.4$ & $75.0$\dgain{0.6} & $62.3$ & $66.5$\dgain{4.2} \\
            & $0.9$ & $90.8$ & $91.4$\dgain{0.6} & $70.0$ & $72.0$\dgain{2.0} & $75.0$ & $75.0$\dnil & $66.9$ & $68.5$\dgain{1.6} \\
        \bottomrule
    \end{tabular}
    \vspace{0.5em}
    \caption{\textbf{MRP remasking improves the accuracy of low-threshold dynamic decoding.}
    For each SDAR backbone (1.7B/4B/8B) and unmasking threshold $\tau$, we report task
    accuracy (\%, $\uparrow$) for plain threshold-based \emph{Dynamic} decoding (\emph{Dyn.};
    our baseline) and for \emph{MRP Remask} (\emph{Remask}; our approach), which reuses the
    MRP-corrected confidence to remask over-eager reveals with the \emph{same} $\tau$ for both
    the reveal and the remask (no extra tuning). Next to each Remask value we report its
    accuracy change over Dynamic at matched $\tau$ (\textcolor{betterGreen}{green}: gain,
    \textcolor{worseRed}{red}: loss).}
    \label{tab:remask-acc}
\end{table}

\textbf{Remasking recovers the accuracy lost to aggressive unmasking.} At low thresholds the backbone commits many positions in a single step, before their neighbors are revealed, and Dynamic accuracy degrades sharply as a result. MRP Remask reverses most of this loss: at $\tau = 0.5$ it improves accuracy by at least $11$ points at every task and scale, and by up to $+22.6$ (e.g., $32.3 \to 54.9$ on 8B HumanEval, $63.4 \to 81.1$ on 4B GSM8K, $26.0 \to 37.4$ on 1.7B MATH500). The improvement shrinks as $\tau$ rises: by $\tau = 0.9$ the backbone already unmasks conservatively, leaving little for MRP remasking to revoke, and the gain narrows to near zero. MRP remasking thus acts exactly where it is needed---when the backbone is over-eager---and stays out of the way otherwise.

\textbf{Reasoning improves uniformly; the rare regressions are small and confined to code.} On the reasoning benchmarks (GSM8K, MATH500), every operating point improves, without a single exception across thresholds or model scales, indicating that the inter-step residual reliably identifies premature reveals for natural-language reasoning. On code generation the gains are equally large at aggressive thresholds, but a few minor regressions surface at conservative thresholds (e.g., $-1.8$ and $-1.2$ on 4B HumanEval at $\tau = 0.8$ and $0.9$, and $-0.8$ on 4B MBPP at $\tau = 0.9$). Code tokens undergo sharper, less predictable distributional shifts between denoising steps, which makes the residual harder to estimate; where Dynamic is already near its ceiling, an imperfect estimate can occasionally revoke a correct reveal. The effect is small and limited to the high-$\tau$ regime where remasking is least needed.

\textbf{The benefit is consistent across model scale.} All three backbones follow the same profile---large gains at low $\tau$ that taper to near zero at high $\tau$---so remasking does not wash out as the backbone strengthens: even the 8B model gains $+14.4$ on GSM8K and $+22.6$ on HumanEval at $\tau = 0.5$. In effect, remasking lets a practitioner run at an aggressive, low-latency threshold while retaining most of the accuracy of a conservative one, at the cost of a single additional MRP forward per step; we report the corresponding decoding throughput in Table~\ref{tab:remask-tps} (Appendix~\ref{app:remask-throughput}).

\vspace{-0.5em}
\subsection{Scaling Law}
\label{sec:exp:scaling}

\vspace{-0.5em}

Holding the backbone fixed at SDAR-1.7B, we vary the number of MRP transformer
layers $D \in \{1, 2, 3, 4, 8\}$. A deeper MRP module has more capacity to
predict the residual but adds non-trivial cost per inference step.
Figure~\ref{fig:ablation-depth} reports accuracy and throughput across the four
benchmarks for both \emph{direct} ($K{\in}\{1,2\}$) and \emph{speculative}
($K{\in}\{2,3\}$) decoding modes.

\begin{figure}[t!]
    \centering
    \includegraphics[width=\textwidth]{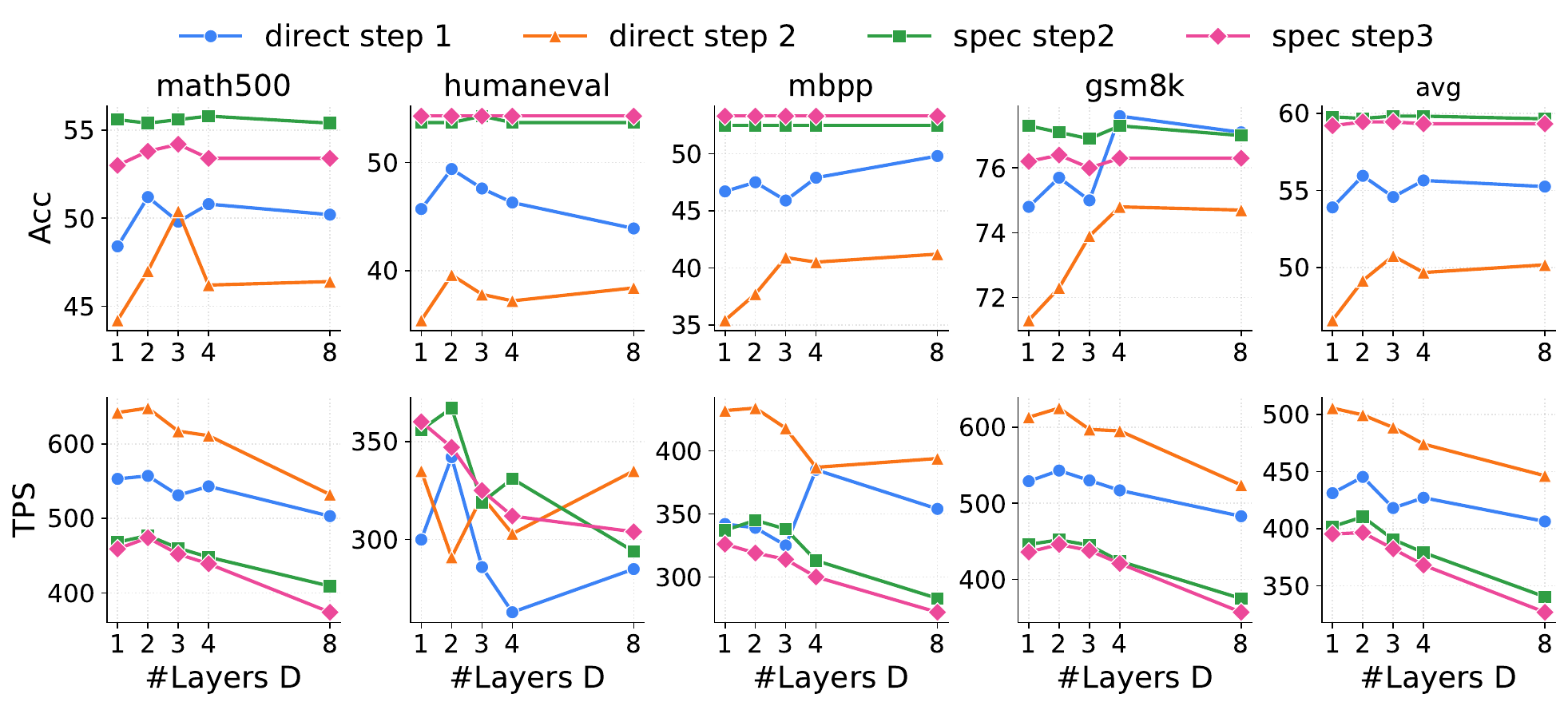}
    \caption{Effect of MRP depth on SDAR-1.7B. We sweep the number of MRP
    transformer layers $D \in \{1, 2, 3, 4, 8\}$ on the $x$-axis and report,
    per benchmark column (and an averaged column at the right), accuracy on
    the top row and decoding throughput on the bottom row.
    Each subplot draws four curves corresponding to the four decoding
    configurations: \emph{direct} with $K{\in}\{1,2\}$ MRP steps and
    \emph{speculative} with $K{\in}\{2,3\}$ MRP steps.
    The underlying numerical values are reported in
    Table~\ref{tab:layer-sweep-1p7b} (Appendix~\ref{app:layer-sweep}).}
    \label{fig:ablation-depth}
\end{figure}

\textbf{A small module is enough.} Adding MRP layers improves direct-decoding accuracy with diminishing returns: the gains are clearest for $K{=}2$ (the four-benchmark average rises by about four points from $D{=}1$ to $D{=}3$) and on GSM8K (where direct accuracy keeps climbing to a peak at $D{=}4$), but for $K{=}1$ they are small and non-monotonic, and speculative accuracy is essentially flat across depth. Throughput, by contrast, falls steadily as each added layer increases the per-step latency of the MRP forward pass, so the best speed--quality tradeoff sits at small depth---a few layers recover essentially all of the achievable accuracy, and the 8-layer variant only trades throughput for no meaningful accuracy gain. This is consistent with our theoretical characterization: the inter-step residual has limited complexity, and a small module suffices to capture the bulk of the signal. HumanEval is an exception: its short, single-function generations decode comparatively few tokens per problem, yielding the least stable throughput estimates across depths.

\textbf{The optimal depth depends on the inference mode.} In direct decoding, where MRP predictions are used without verification, deeper modules improve quality more noticeably: the accuracy gap between 1-layer and 3-layer MRP is larger than in speculative mode. In speculative decoding, verification corrects MRP errors, so a shallower (and faster) module can achieve comparable final quality while maintaining higher throughput. This suggests a practical guideline: use deeper MRP for direct decoding when quality is prioritized, and shallower MRP for speculative decoding when lossless acceleration is the goal.

\vspace{-0.5em}
\section{Conclusion and Future Work}
\label{sec:conclusion}

We introduced Multi-token Residual Prediction (MRP), a lightweight module that accelerates diffusion language model inference by predicting the logit residual between adjacent denoising steps. MRP supports three inference modes: direct decoding for a tunable quality--speed tradeoff, speculative decoding for lossless acceleration, and remasking decoding for improved accuracy under aggressive low-threshold decoding. Experiments on SDAR models at three scales demonstrate up to $1.40\times$ lossless speedup in SGLang, with consistent gains across reasoning and code generation benchmarks. All training data, code, and model weights are publicly available.

Several directions remain open. First, our dynamic decoding results show that MRP is most effective under conservative thresholds, where the backbone unmasks few tokens per step. Co-designing the threshold schedule with MRP, for instance using a more aggressive base threshold while relying on MRP's predictions to recover quality, could unlock stronger acceleration without sacrificing accuracy. Second, the current MRP module uses a fixed number of residual steps $K$ throughout generation. An adaptive schedule that allocates more MRP steps when the backbone is likely to unmask few tokens, and skips MRP entirely when the backbone is already making rapid progress, would better match compute to difficulty. Third, while MRP performs robustly on reasoning tasks, we observe larger quality degradation on code generation, suggesting that the inter-step residual structure differs across domains. Understanding and addressing this gap, potentially through domain-aware training or architecture modifications, is an important next step. 

\section*{Acknowledgements}
This work was supported in part through STCSM 25PJA108 and NYU IT High Performance Computing resources, services, and staff expertise. We gratefully acknowledge Nous Research and Modal for providing the GPU compute.

\iftoggle{arxiv}{
\printbibliography
}{
\bibliography{biblio}

@inproceedings{austin2021D3PM,
  author    = {Jacob Austin and Daniel D. Johnson and Jonathan Ho and Daniel Tarlow and Rianne van den Berg},
  title     = {Structured Denoising Diffusion Models in Discrete State-Spaces},
  booktitle = {NeurIPS},
  year      = {2021}
}

@inproceedings{lou2024SEDD,
  author    = {Aaron Lou and Chenlin Meng and Stefano Ermon},
  title     = {Discrete Diffusion Modeling by Estimating the Ratios of the Data Distribution},
  booktitle = {ICML},
  year      = {2024}
}

@inproceedings{sahoo2024MDLM,
  author    = {Subham Sekhar Sahoo and Marianne Arriola and Yair Schiff and Aaron Gokaslan and Edgar Marroquin and Justin T. Chiu and Alexander Rush and Volodymyr Kuleshov},
  title     = {Simple and Effective Masked Diffusion Language Models},
  booktitle = {NeurIPS},
  year      = {2024}
}

@inproceedings{arriola2025BD3LM,
  author    = {Marianne Arriola and Aaron Gokaslan and Justin T. Chiu and Zhihan Yang and Zhixuan Qi and Jiaqi Han and Subham Sekhar Sahoo and Volodymyr Kuleshov},
  title     = {Block Diffusion: Interpolating Between Autoregressive and Diffusion Language Models},
  booktitle = {ICLR},
  year      = {2025}
}

@inproceedings{nie2025LLaDA,
  author    = {Shen Nie and Fengqi Zhu and Zebin You and Xiaolu Zhang and Jingyang Ou and Jun Hu and Jun Zhou and Yankai Lin and Ji-Rong Wen and Chongxuan Li},
  title     = {Large Language Diffusion Models},
  booktitle = {NeurIPS},
  year      = {2025}
}

@article{ye2025Dream,
  author  = {Jiacheng Ye and Zhihui Xie and Lin Zheng and Jiahui Gao and Zirui Wu and Xin Jiang and Zhenguo Li and Lingpeng Kong},
  title   = {Dream 7B: Diffusion Large Language Models},
  journal = {arXiv preprint arXiv:2508.15487},
  year    = {2025}
}

@misc{teknium2024OpenHermes,
  author       = {Teknium},
  title        = {{OpenHermes 2.5}: An Open Dataset of Synthetic Data for Generalist {LLM} Assistants},
  year         = {2024},
  howpublished = {\url{https://huggingface.co/datasets/teknium/OpenHermes-2.5}}
}

@article{chen2024GenQA,
  author  = {Jiuhai Chen and Rifaa Qadri and Yuxin Wen and Neel Jain and John Kirchenbauer and Tianyi Zhou and Tom Goldstein},
  title   = {{GenQA}: Generating Millions of Instructions from a Handful of Prompts},
  journal = {arXiv preprint arXiv:2406.10323},
  year    = {2024}
}

@misc{baai2024InfinityInstruct,
  author       = {{BAAI}},
  title        = {{Infinity Instruct}: Scaling Instruction Selection and Synthesis to Enhance Language Models},
  year         = {2024},
  howpublished = {\url{https://huggingface.co/datasets/BAAI/Infinity-Instruct}}
}

@inproceedings{wang2024MambaInLlama,
  author    = {Junxiong Wang and Daniele Paliotta and Avner May and Alexander M. Rush and Tri Dao},
  title     = {The {Mamba} in the {Llama}: Distilling and Accelerating Hybrid Models},
  booktitle = {NeurIPS},
  year      = {2024}
}

@article{cheng2025SDAR,
  author  = {Shuang Cheng and Yihan Bian and Dawei Liu and Linfeng Zhang and Qian Yao and Zhongbo Tian and Wenhai Wang and Qipeng Guo and Kai Chen and Biqing Qi and Bowen Zhou},
  title   = {{SDAR}: A Synergistic Diffusion-AutoRegression Paradigm for Scalable Sequence Generation},
  journal = {arXiv preprint arXiv:2510.06303},
  year    = {2025}
}

@inproceedings{wang2025DiffusionForcing,
  author    = {Xu Wang and Chenkai Xu and Yijie Jin and Jiachun Jin and Hao Zhang and Zhijie Deng},
  title     = {Diffusion {LLMs} Can Do Faster-Than-{AR} Inference via Discrete Diffusion Forcing},
  booktitle = {ICLR},
  year      = {2026}
}

@inproceedings{wu2025FastdLLM,
  author    = {Chengyue Wu and Hao Zhang and Shuchen Xue and Zhijian Liu and Shizhe Diao and Ligeng Zhu and Ping Luo and Song Han and Enze Xie},
  title     = {Fast-d{LLM}: Training-free Acceleration of Diffusion {LLM} by Enabling {KV} Cache and Parallel Decoding},
  booktitle = {ICLR},
  year      = {2026}
}

@article{hong2025WINO,
  author  = {Feng Hong and Geng Yu and Yushi Ye and Haicheng Huang and Huangjie Zheng and Ya Zhang and Yanfeng Wang and Jiangchao Yao},
  title   = {Wide-In, Narrow-Out: Revokable Decoding for Efficient and Effective {DLLMs}},
  journal = {arXiv preprint arXiv:2507.18578},
  year    = {2025}
}

@inproceedings{li2025Prophet,
  author    = {Pengxiang Li and Yefan Zhou and Dilxat Muhtar and Lu Yin and Shilin Yan and Li Shen and Soroush Vosoughi and Shiwei Liu},
  title     = {Diffusion Language Models Know the Answer Before Decoding},
  booktitle = {NeurIPS},
  year      = {2025}
}

@article{mohamed2025SchED,
  author  = {Amr Mohamed and Yang Zhang and Michalis Vazirgiannis and Guokan Shang},
  title   = {Fast-Decoding Diffusion Language Models via Progress-Aware Confidence Schedules},
  journal = {arXiv preprint arXiv:2512.02892},
  year    = {2025}
}

@article{agrawal2025Spiffy,
  author  = {Sudhanshu Agrawal and Risheek Garrepalli and Raghavv Goel and Mingu Lee and Christopher Lott and Fatih Porikli},
  title   = {Spiffy: Multiplying Diffusion {LLM} Acceleration via Lossless Speculative Decoding},
  journal = {arXiv preprint arXiv:2509.18085},
  year    = {2025}
}

@inproceedings{christopher2025SpecDiff,
  author    = {Jacob K. Christopher and Brian R. Bartoldson and Tal Ben-Nun and Michael Cardei and Bhavya Kailkhura and Ferdinando Fioretto},
  title     = {Speculative Diffusion Decoding: Accelerating Language Generation through Diffusion},
  booktitle = {NAACL},
  year      = {2025}
}

@inproceedings{israel2025APD,
  author    = {Daniel Israel and Guy Van den Broeck and Aditya Grover},
  title     = {Accelerating Diffusion {LLMs} via Adaptive Parallel Decoding},
  booktitle = {NeurIPS},
  year      = {2025}
}

@inproceedings{seo2025Conv,
  author    = {Yeongbin Seo and Dongha Lee and Jaehyung Kim and Jinyoung Yeo},
  title     = {Fast and Fluent Diffusion Language Models via Convolutional Decoding and Rejective Fine-tuning},
  booktitle = {NeurIPS},
  year      = {2025}
}

@inproceedings{bao2026Learn2PD,
  author    = {Wenrui Bao and Zhiben Chen and Dan Xu and Yuzhang Shang},
  title     = {Learning to Parallel: Accelerating Diffusion Large Language Models via Learnable Parallel Decoding},
  booktitle = {ICLR},
  year      = {2026}
}

@article{li2025DiffuSpec,
  author  = {Guanghao Li and Zhihui Fu and Min Fang and Qibin Zhao and Ming Tang and Chun Yuan and Jun Wang},
  title   = {Diffu{S}pec: Unlocking Diffusion Language Models for Speculative Decoding},
  journal = {arXiv preprint arXiv:2510.02358},
  year    = {2025}
}

@article{gao2025SSD,
  author  = {Yifeng Gao and Ziang Ji and Yuxuan Wang and Biqing Qi and Hanlin Xu and Linfeng Zhang},
  title   = {Self Speculative Decoding for Diffusion Large Language Models},
  journal = {arXiv preprint arXiv:2510.04147},
  year    = {2025}
}

@inproceedings{brown2020GPT3,
  author    = {Tom Brown and Benjamin Mann and Nick Ryder and Melanie Subbiah and Jared D. Kaplan and Prafulla Dhariwal and Arvind Neelakantan and Pranav Shyam and Girish Sastry and Amanda Askell and others},
  title     = {Language Models are Few-Shot Learners},
  booktitle = {NeurIPS},
  year      = {2020}
}

@article{yang2025qwen3,
  author  = {An Yang and Anfeng Li and Baosong Yang and Beichen Zhang and Binyuan Hui and Bo Zheng and Bowen Yu and Chang Gao and Chengen Huang and Chenxu Lv and others},
  title   = {{Qwen3} Technical Report},
  journal = {arXiv preprint arXiv:2505.09388},
  year    = {2025}
}

@article{liu2024DeepSeekV3,
  author  = {Aixin Liu and Bei Feng and Bing Xue and Bingxuan Wang and Bochao Wu and Chengda Lu and Chenggang Zhao and Chengqi Deng and Chenyu Zhang and Chong Ruan and others},
  title   = {{DeepSeek-V3} Technical Report},
  journal = {arXiv preprint arXiv:2412.19437},
  year    = {2024}
}

@inproceedings{zheng2024sglang,
  author    = {Lianmin Zheng and Liangsheng Yin and Zhiqiang Xie and Chuyue Sun and Jeff Huang and Cody H. Yu and Shiyi Cao and Christos Kozyrakis and Ion Stoica and Joseph E. Gonzalez and others},
  title     = {{SGLang}: Efficient Execution of Structured Language Model Programs},
  booktitle = {NeurIPS},
  year      = {2024}
}

@article{li2024eagle,
  author  = {Yuhui Li and Fangyun Wei and Chao Zhang and Hongyang Zhang},
  title   = {{EAGLE}: Speculative Sampling Requires Rethinking Feature Uncertainty},
  journal = {arXiv preprint arXiv:2401.15077},
  year    = {2024}
}

@article{nair2025softmaxLipschitz,
  author  = {Pravin Nair},
  title   = {Softmax is {$1/2$}-{L}ipschitz: A Tight Bound Across All {$\ell_p$} Norms},
  journal = {arXiv preprint arXiv:2510.23012},
  year    = {2025}
}

@article{cobbe2021GSM8K,
  author  = {Karl Cobbe and Vineet Kosaraju and Mohammad Bavarian and Mark Chen and Heewoo Jun and Lukasz Kaiser and Matthias Plappert and Jerry Tworek and Jacob Hilton and Reiichiro Nakano and Christopher Hesse and John Schulman},
  title   = {Training Verifiers to Solve Math Word Problems},
  journal = {arXiv preprint arXiv:2110.14168},
  year    = {2021}
}

@inproceedings{hendrycks2021MATH,
  author    = {Dan Hendrycks and Collin Burns and Saurav Kadavath and Akul Arora and Steven Basart and Eric Tang and Dawn Song and Jacob Steinhardt},
  title     = {Measuring Mathematical Problem Solving with the {MATH} Dataset},
  booktitle = {NeurIPS Datasets and Benchmarks},
  year      = {2021}
}

@inproceedings{lightman2024PRM800K,
  author    = {Hunter Lightman and Vineet Kosaraju and Yura Burda and Harri Edwards and Bowen Baker and Teddy Lee and Jan Leike and John Schulman and Ilya Sutskever and Karl Cobbe},
  title     = {Let's Verify Step by Step},
  booktitle = {ICLR},
  year      = {2024}
}

@article{chen2021HumanEval,
  author  = {Mark Chen and Jerry Tworek and Heewoo Jun and Qiming Yuan and Henrique Pondé de Oliveira Pinto and Jared Kaplan and Harri Edwards and Yuri Burda and Nicholas Joseph and Greg Brockman and others},
  title   = {Evaluating Large Language Models Trained on Code},
  journal = {arXiv preprint arXiv:2107.03374},
  year    = {2021}
}

@article{austin2021MBPP,
  author  = {Jacob Austin and Augustus Odena and Maxwell Nye and Maarten Bosma and Henryk Michalewski and David Dohan and Ellen Jiang and Carrie Cai and Michael Terry and Quoc Le and Charles Sutton},
  title   = {Program Synthesis with Large Language Models},
  journal = {arXiv preprint arXiv:2108.07732},
  year    = {2021}
}

@article{hu2026RCDLM,
  author  = {Yuezhou Hu and Harman Singh and Monishwaran Maheswaran and Haocheng Xi and Coleman Hooper and Jintao Zhang and Aditya Tomar and Michael W. Mahoney and Sewon Min and Mehrdad Farajtabar and others},
  title   = {Residual Context Diffusion Language Models},
  journal = {arXiv preprint arXiv:2601.22954},
  year    = {2026}
}

@article{wu2025FastdLLMv2,
  author  = {Chengyue Wu and Hao Zhang and Shuchen Xue and Shizhe Diao and Yonggan Fu and Zhijian Liu and Pavlo Molchanov and Ping Luo and Song Han and Enze Xie},
  title   = {Fast-d{LLM} v2: Efficient Block-Diffusion {LLM}},
  journal = {arXiv preprint arXiv:2509.26328},
  year    = {2025}
}

@inproceedings{he2016ResNet,
  author    = {Kaiming He and Xiangyu Zhang and Shaoqing Ren and Jian Sun},
  title     = {Deep Residual Learning for Image Recognition},
  booktitle = {CVPR},
  year      = {2016}
}
}

\newpage

\appendix

\section{Related Work}
\label{sec:related}

\subsection{Diffusion Language Models}
\label{sec:related-dlm}

Discrete-sequence diffusion was formalized by D3PM~\citep{austin2021D3PM}, whose absorbing-state ($[\texttt{MASK}]$) variant seeded masked diffusion language modeling. SEDD~\citep{lou2024SEDD} and MDLM~\citep{sahoo2024MDLM} subsequently sharpened the training objective, with MDLM's Rao-Blackwellized loss reducing masked diffusion training to a weighted mixture of masked language modeling losses. This recipe enabled the first wave of large-scale DLMs. LLaDA~\citep{nie2025LLaDA} scaled masked diffusion to 8B parameters and matched strong AR LLMs on broad reasoning benchmarks. BD3-LM~\citep{arriola2025BD3LM} introduced a block-causal design that generates tokens block-by-block with discrete diffusion within each block, addressing fixed-length generation and the absence of KV caching. This paradigm has since become dominant, with Dream 7B~\citep{ye2025Dream}, D2F~\citep{wang2025DiffusionForcing}, and most recently SDAR~\citep{cheng2025SDAR} further validating the recipe at scale. SDAR is particularly relevant to our work: it lightly adapts a pre-trained AR backbone into a block-diffusion sampler, and serves as the backbone for all experiments in this paper. Our contribution is orthogonal to these efforts: we accelerate \emph{inference} of any such backbone rather than proposing a new training paradigm.

\subsection{Fast Decoding for Diffusion Language Models}
\label{sec:related-fast}

As discussed in Section~\ref{sec:intro}, naively unmasking many tokens per step degrades quality because the unmasking decisions at different positions are made simultaneously from a single forward pass. Existing approaches to accelerating DLM inference can be organized into two families: \emph{threshold-based} methods, which use the backbone's own per-token confidence as the unmasking criterion, and \emph{speculative} methods, which introduce auxiliary mechanisms such as drafting, learned predictors, or structural priors to decode additional tokens beyond what confidence alone would permit.

\paragraph{Threshold-based methods.}
Fast-dLLM~\citep{wu2025FastdLLM} established the standard recipe: at each denoising step, unmask every position whose top-1 confidence exceeds a fixed threshold $\tau$, and enable a KV cache for block-diffusion models. Subsequent work refines this recipe along several axes. WINO~\citep{hong2025WINO} allows the bidirectional context to re-mask previously committed tokens in later steps, relaxing the irrevocable commitment of Fast-dLLM. Prophet~\citep{li2025Prophet} takes the opposite approach, committing all remaining tokens at once when a global confidence margin indicates that further refinement would be wasteful. SchED~\citep{mohamed2025SchED} generalizes the static threshold to a progress-aware schedule $\tau(t)$ that aggregates full-span logit margins. Despite these advances, all threshold-based methods rely on backbone confidence as the sole unmasking signal, and therefore cannot decode tokens whose confidence has not yet crossed the threshold, even when the inter-step residual suggests they are already close to their final value.

\paragraph{Speculative methods.}
A complementary line of work introduces explicit drafting mechanisms. SpecDiff~\citep{christopher2025SpecDiff} uses a discrete diffusion model as a drafter inside a speculative decoding loop with an autoregressive verifier, while DiffuSpec~\citep{li2025DiffuSpec} reverses this arrangement, using a DLM to draft for AR verifiers via causal-consistency path search. Within the DLM family itself, Spiffy~\citep{agrawal2025Spiffy} drafts block states from the model's own distribution and verifies them through a directed draft graph, and SSD~\citep{gao2025SSD} eliminates the auxiliary drafter entirely by using the backbone as both drafter and hierarchical-tree verifier in a single forward pass. Learn2PD~\citep{bao2026Learn2PD} takes a different approach, training a lightweight filter that predicts whether each draft token already matches its final committed value, avoiding unnecessary recomputation. APD~\citep{israel2025APD} and Conv/R2FT~\citep{seo2025Conv} explore orthogonal directions: APD modulates the joint decoding distribution, while Conv/R2FT exploits positional locality through a convolutional decoding window over the masked positions.

\paragraph{Position of our work.}
MRP sits at the intersection of these two families. Like threshold-based methods, it uses the backbone's confidence for verification, preserving plug-in compatibility with any block-diffusion backbone. Like speculative methods, it generates additional token candidates beyond what the threshold alone would yield, through a chain of lightweight residual heads whose proposals the backbone then accepts or rejects. The key distinction from prior speculative approaches such as SSD~\citep{gao2025SSD} and Learn2PD~\citep{bao2026Learn2PD} is MRP's \emph{residual} parameterization: rather than predicting the full next-step distribution or learning a binary accept/reject filter, MRP predicts only the small correction between adjacent denoising steps, which we show in Section~\ref{sec:method} is decisive for keeping the module lightweight while matching backbone quality at high parallelism.

\section{Residual Magnitude Measurement (Figure~\ref{fig:motivation})}
\label{app:residual_magnitude}

This appendix specifies the protocol used to produce
Figure~\ref{fig:motivation}, which compares the magnitude of the per-step
residual against the magnitude of the underlying backbone state.

\textbf{Models and data.} We measure on three frozen SDAR-Chat
backbones~\citep{cheng2025SDAR} at the 1.7B, 4B, and 8B scales (block size
$B{=}16$). Prompts are sampled from the GSM8K test
set~\citep{cobbe2021GSM8K} with a fixed seed; questions are wrapped in the
SDAR chat template with the prefix \emph{``Let's think step by step.''} and
prompts longer than $1024$ tokens are skipped. We collect statistics from
the first four decoded blocks of each prompt and stop after at least
$200$ blocks per model.

\textbf{Decoding.} Each block of $16$ positions is decoded with greedy,
low-confidence-static unmasking, revealing exactly one token per
denoising step (the position with the highest argmax-softmax confidence
among the still-masked positions). Decoding stops early on EOS tokens.
This procedure traverses $17$ block states $s = 0, 1, \ldots, 16$, where
$s$ counts the number of revealed positions in the current block.

\textbf{What we record.} At every state $s$ we save both the backbone
final-layer hidden states $\*h^{(s)} \in \mathbb{R}^{B \times d}$ and the
post-projection logits $\*\ell^{(s)} \in \mathbb{R}^{B \times V}$,
restricted to the current block's positions. For brevity we let
$X^{(s)} \in \{\*\ell^{(s)}, \*h^{(s)}\}$ denote either tensor below.

\textbf{Residual and reference quantities.} For each $r \in \{1, \ldots,
16\}$ and each valid offset $t$, we compute the residual
\begin{equation*}
    \*\delta_{t,r} \;=\; X^{(t+r)} - X^{(t)},
    \qquad t = 0, 1, \ldots, 16-r,
\end{equation*}
which is the change in $X$ produced by revealing $r$ additional tokens.
The backbone reference is taken over the partially decoded states
$X^{(s)}$ for $s = 1, \ldots, 16$.

\textbf{Magnitude metric.} Throughout the figure we report the
root-mean-square magnitude per entry,
$\|Y\|_{\mathrm{RMS}} = \sqrt{\frac{1}{|Y|} \sum_{i} Y_i^2}$,
computed in float64 to avoid bf16 underflow. For each model, space
(logits or hidden states), and value of $r$, we average
$\|\*\delta_{t,r}\|_{\mathrm{RMS}}$ over all collected $(\text{block},\,
t)$ pairs; the dashed reference lines show the average of
$\|X^{(s)}\|_{\mathrm{RMS}}$ over all collected $(\text{block},\, s)$
pairs.

\section{Implementation Details}
\label{app:hyperparams}

Table~\ref{tab:hyperparams} lists the full set of training hyperparameters
used for all three model scales.

\begin{table}[t!]
    \centering
    \tabfont
    \begin{tabular}{ll}
        \toprule
        \textbf{Hyperparameter} & \textbf{Value} \\
        \midrule
        \multicolumn{2}{l}{\emph{Architecture}} \\
        Backbone (frozen)              & SDAR-\{1.7B, 4B, 8B\}-Chat (block size 16) \\
        MRP transformer layers $D$     & 3 \\
        Frozen modules                 & backbone, LM head, token embeddings \\
        MRP init.\ std.\ $\sigma_{\text{init}}$ & 0.2 \\
        \midrule
        \multicolumn{2}{l}{\emph{Objective}} \\
        Loss                           & forward KD on residual logits \\
        KD temperature $T_{\mathrm{KD}}$ & 1.0 \\
        Reveal mode                    & ground-truth, $k{=}1$ token per block \\
        MRP unroll steps $K$           & 2 \\
        Per-step loss weighting        & uniform \\
        \midrule
        \multicolumn{2}{l}{\emph{Optimization}} \\
        Optimizer                      & AdamW \\
        Peak learning rate             & $1\mathrm{e}{-3}$ \\
        LR schedule                    & cosine with min-LR \\
        Epochs                         & 1 \\
        Sequence length                & 4096 \\
        Block length                   & 16 \\
        Global / micro batch size      & 16 / 1 \\
        Gradient checkpointing         & enabled \\
        \midrule
        \multicolumn{2}{l}{\emph{Data}} \\
        Dataset                        & SFTDatasetV3~\citep{wang2024MambaInLlama} \\
        Supervised tokens              & assistant response tokens only \\
        \bottomrule
    \end{tabular}
    \vspace{0.5em}
    \caption{Training hyperparameters used for all three model scales.}
    \label{tab:hyperparams}
\end{table}

\section{Throughput of MRP Remasking}
\label{app:remask-throughput}

Table~\ref{tab:remask-tps} reports decoding throughput (tokens per second, measured in
SGLang on a single NVIDIA H100) for \emph{Dynamic} decoding and \emph{MRP Remask}, over the
same backbones and thresholds as the accuracy results in Table~\ref{tab:remask-acc}.
Remasking adds one MRP forward per step and commits fewer tokens per step (the revoked
reveals), so it generally runs somewhat slower than Dynamic at matched $\tau$: on the
reasoning benchmarks it retains roughly $0.76$--$0.88\times$ of the Dynamic throughput, while
on the code benchmarks the ratio is more variable, ranging from modestly slower to
occasionally faster. This throughput reduction is the efficiency cost of the accuracy gains
in Table~\ref{tab:remask-acc}.

\begin{table}[t]
    \centering
    \tabfont
    \begin{tabular}{cl ll ll ll ll}
        \toprule
        \multirow{2}{*}[-0.5ex]{\textbf{Model}} & \multirow{2}{*}[-0.5ex]{\textbf{$\tau$}}
            & \multicolumn{2}{c}{\textbf{GSM8K}}
            & \multicolumn{2}{c}{\textbf{MATH500}}
            & \multicolumn{2}{c}{\textbf{HumanEval}}
            & \multicolumn{2}{c}{\textbf{MBPP}} \\
        \cmidrule(lr){3-4} \cmidrule(lr){5-6} \cmidrule(lr){7-8} \cmidrule(lr){9-10}
            & & Dyn. & Remask & Dyn. & Remask & Dyn. & Remask & Dyn. & Remask \\
        \midrule
        \multirow{5}{*}{1.7B}
            & $0.5$ & $975$ & $792$ & $1126$ & $956$ & $588$ & $587$ & $533$ & $363$ \\
            & $0.6$ & $881$ & $672$ & $1065$ & $871$ & $554$ & $528$ & $420$ & $378$ \\
            & $0.7$ & $820$ & $651$ & $985$ & $818$ & $462$ & $555$ & $506$ & $389$ \\
            & $0.8$ & $719$ & $565$ & $929$ & $761$ & $528$ & $556$ & $420$ & $337$ \\
            & $0.9$ & $627$ & $491$ & $784$ & $639$ & $480$ & $512$ & $375$ & $297$ \\
        \midrule
        \multirow{5}{*}{4B}
            & $0.5$ & $720$ & $558$ & $753$ & $603$ & $448$ & $354$ & $441$ & $309$ \\
            & $0.6$ & $641$ & $513$ & $728$ & $593$ & $449$ & $437$ & $269$ & $297$ \\
            & $0.7$ & $596$ & $489$ & $660$ & $532$ & $433$ & $409$ & $253$ & $214$ \\
            & $0.8$ & $523$ & $422$ & $599$ & $490$ & $400$ & $400$ & $235$ & $206$ \\
            & $0.9$ & $466$ & $366$ & $519$ & $437$ & $318$ & $363$ & $211$ & $179$ \\
        \midrule
        \multirow{5}{*}{8B}
            & $0.5$ & $507$ & $416$ & $532$ & $469$ & $354$ & $324$ & $253$ & $199$ \\
            & $0.6$ & $454$ & $378$ & $524$ & $432$ & $361$ & $342$ & $229$ & $180$ \\
            & $0.7$ & $419$ & $340$ & $479$ & $394$ & $370$ & $334$ & $210$ & $166$ \\
            & $0.8$ & $378$ & $306$ & $439$ & $349$ & $309$ & $298$ & $187$ & $154$ \\
            & $0.9$ & $313$ & $257$ & $380$ & $312$ & $326$ & $277$ & $167$ & $134$ \\
        \bottomrule
    \end{tabular}
    \vspace{0.5em}
    \caption{\textbf{Decoding throughput (tokens/s) of MRP Remask vs.\ Dynamic decoding.}
    Measured in SGLang on a single NVIDIA H100, for each SDAR backbone (1.7B/4B/8B) and
    unmasking threshold $\tau$. \emph{Dyn.}\ is plain threshold-based dynamic decoding and
    \emph{Remask} is our MRP remasking mode. This is the throughput counterpart to the
    accuracy results in Table~\ref{tab:remask-acc}.}
    \label{tab:remask-tps}
\end{table}

\section{MRP Depth Sweep on SDAR-1.7B}
\label{app:layer-sweep}

Table~\ref{tab:layer-sweep-1p7b} reports the raw accuracy and decoding
throughput numbers underlying Figure~\ref{fig:ablation-depth}. The MRP
module depth is varied over $D \in \{1, 2, 3, 4, 8\}$ on top of a frozen
SDAR-1.7B-Chat backbone, and we evaluate the four decoding configurations
shown in the figure: \emph{direct} with $K{\in}\{1, 2\}$ MRP steps and
\emph{speculative} with $K{\in}\{2, 3\}$ MRP steps. All other settings
follow Table~\ref{tab:hyperparams}; throughput is reported in tokens per
second on a single H100, and accuracy is the standard task metric (\%).

\section{Proofs and Derivations}
\label{app:proofs}

We provide complete proofs for the results stated in Section~\ref{sec:method}. Throughout, we use the notation established in Section~\ref{sec:prelim}: $\*x_t$ is the partially masked sequence at step $t$, $f$ is the backbone, $\*\ell_t = f(\*x_t) \in \mathbb{R}^{L \times V}$ are the logits, $\*\pi_t^i = \mathrm{softmax}(\*\ell_t^i) \in \Delta^V$ is the predictive distribution at position $i$, $\mathcal{R}_t$ is the set of positions unmasked at step $t$, and $\*e_v \in \mathbb{R}^{d_e}$ denotes the embedding of token $v$.

\subsection{Proof of Proposition~\ref{prop:contraction}}

\begin{proof}
The transition $\*x_t \to \*x_{t-1}$ replaces the \texttt{[MASK]} token at each position $j \in \mathcal{R}_t$ with the predicted token $v_j$, leaving all other positions unchanged. In embedding space, the perturbation $\*E_{t-1} - \*E_t$ is nonzero at exactly $|\mathcal{R}_t|$ rows, with Frobenius norm:
\begin{equation}
    \|\*E_{t-1} - \*E_t\|_F = \sqrt{\sum_{j \in \mathcal{R}_t} \|\*e_{v_j} - \*e_{\texttt{[MASK]}}\|_2^2} \leq \sqrt{|\mathcal{R}_t|} \cdot \max_{j \in \mathcal{R}_t} \|\*e_{v_j} - \*e_{\texttt{[MASK]}}\|_2.
    \label{eq:perturbation_norm}
\end{equation}

Since $f$ is $\kappa$-Lipschitz with respect to the input embedding matrix in Frobenius norm, and the Lipschitz condition applies to the full output tensor $\*\ell \in \mathbb{R}^{L \times V}$ while we bound a single row, the per-position logit change satisfies:
\begin{equation}
    \|\*\ell_{t-1}^i - \*\ell_t^i\|_2 \leq \frac{\kappa\sqrt{|\mathcal{R}_t|}}{L} \cdot \max_{j \in \mathcal{R}_t} \|\*e_{v_j} - \*e_{\texttt{[MASK]}}\|_2.
    \label{eq:logit_bound}
\end{equation}

To convert this into a total variation bound, we invoke the fact that softmax is $\frac{1}{2}$-Lipschitz in all $\ell_p$ norms~\citep{nair2025softmaxLipschitz}. In particular, for the $\ell_2 \to \ell_1$ case:
\begin{equation}
    D_{\mathrm{TV}}(\*\pi_{t-1}^i, \*\pi_t^i) = \frac{1}{2}\|\*\pi_{t-1}^i - \*\pi_t^i\|_1 \leq \frac{1}{2}\|\*\ell_{t-1}^i - \*\ell_t^i\|_2.
\end{equation}

Substituting \eqref{eq:logit_bound} and using $\sqrt{|\mathcal{R}_t|} \leq |\mathcal{R}_t|$ for $|\mathcal{R}_t| \geq 1$, absorbing the factor of $\frac{1}{2}$ into $\kappa$:
\begin{equation}
    D_{\mathrm{TV}}(\*\pi_{t-1}^i, \*\pi_t^i) \leq \kappa \cdot \frac{|\mathcal{R}_t|}{L} \cdot \max_{j \in \mathcal{R}_t} \|\*e_{v_j} - \*e_{\texttt{[MASK]}}\|_2. \qedhere
\end{equation}
\end{proof}

\begin{table}[h!]
    \centering
    \tabfont
    \begin{tabular}{cl cc cc cc cc}
        \toprule
        \multirow{2}{*}[-0.5ex]{\textbf{Depth $D$}} & \multirow{2}{*}[-0.5ex]{\textbf{Setting}}
            & \multicolumn{2}{c}{\textbf{GSM8K}}
            & \multicolumn{2}{c}{\textbf{MATH500}}
            & \multicolumn{2}{c}{\textbf{HumanEval}}
            & \multicolumn{2}{c}{\textbf{MBPP}} \\
        \cmidrule(lr){3-4} \cmidrule(lr){5-6} \cmidrule(lr){7-8} \cmidrule(lr){9-10}
            & & Acc. & TPS & Acc. & TPS & Acc. & TPS & Acc. & TPS \\
        \midrule
        \multirow{4}{*}{$1$}
            & Direct ($K{=}1$)        & $74.8$ & $529$ & $48.4$ & $553$ & $45.7$ & $300$ & $46.7$ & $342$ \\
            & Direct ($K{=}2$)        & $71.3$ & $613$ & $44.2$ & $642$ & $35.4$ & $335$ & $35.4$ & $432$ \\
            & Speculative ($K{=}2$)   & $77.3$ & $446$ & $55.6$ & $468$ & $53.7$ & $356$ & $52.5$ & $337$ \\
            & Speculative ($K{=}3$)   & $76.2$ & $436$ & $53.0$ & $459$ & $54.3$ & $360$ & $53.3$ & $326$ \\
        \midrule
        \multirow{4}{*}{$2$}
            & Direct ($K{=}1$)        & $75.7$ & $543$ & $51.2$ & $557$ & $49.4$ & $342$ & $47.5$ & $339$ \\
            & Direct ($K{=}2$)        & $72.3$ & $625$ & $47.0$ & $648$ & $39.6$ & $291$ & $37.7$ & $434$ \\
            & Speculative ($K{=}2$)   & $77.1$ & $452$ & $55.4$ & $477$ & $53.7$ & $367$ & $52.5$ & $345$ \\
            & Speculative ($K{=}3$)   & $76.4$ & $446$ & $53.8$ & $474$ & $54.3$ & $347$ & $53.3$ & $319$ \\
        \midrule
        \multirow{4}{*}{$3$}
            & Direct ($K{=}1$)        & $75.0$ & $530$ & $49.8$ & $531$ & $47.6$ & $286$ & $45.9$ & $325$ \\
            & Direct ($K{=}2$)        & $73.9$ & $597$ & $50.4$ & $617$ & $37.8$ & $322$ & $40.9$ & $418$ \\
            & Speculative ($K{=}2$)   & $76.9$ & $445$ & $55.6$ & $460$ & $54.3$ & $319$ & $52.5$ & $338$ \\
            & Speculative ($K{=}3$)   & $76.0$ & $438$ & $54.2$ & $452$ & $54.3$ & $325$ & $53.3$ & $314$ \\
        \midrule
        \multirow{4}{*}{$4$}
            & Direct ($K{=}1$)        & $77.6$ & $517$ & $50.8$ & $543$ & $46.3$ & $263$ & $47.9$ & $385$ \\
            & Direct ($K{=}2$)        & $74.8$ & $595$ & $46.2$ & $611$ & $37.2$ & $303$ & $40.5$ & $387$ \\
            & Speculative ($K{=}2$)   & $77.3$ & $424$ & $55.8$ & $448$ & $53.7$ & $331$ & $52.5$ & $313$ \\
            & Speculative ($K{=}3$)   & $76.3$ & $421$ & $53.4$ & $439$ & $54.3$ & $312$ & $53.3$ & $300$ \\
        \midrule
        \multirow{4}{*}{$8$}
            & Direct ($K{=}1$)        & $77.1$ & $483$ & $50.2$ & $503$ & $43.9$ & $285$ & $49.8$ & $354$ \\
            & Direct ($K{=}2$)        & $74.7$ & $524$ & $46.4$ & $532$ & $38.4$ & $335$ & $41.2$ & $394$ \\
            & Speculative ($K{=}2$)   & $77.0$ & $375$ & $55.4$ & $409$ & $53.7$ & $294$ & $52.5$ & $283$ \\
            & Speculative ($K{=}3$)   & $76.3$ & $357$ & $53.4$ & $374$ & $54.3$ & $304$ & $53.3$ & $272$ \\
        \bottomrule
    \end{tabular}
    \vspace{0.5em}
    \caption{Raw numerical results for the SDAR-1.7B MRP depth sweep
    plotted in Figure~\ref{fig:ablation-depth}. Accuracy (Acc., \%) and
    decoding throughput (TPS, tokens per second) are reported per
    benchmark for MRP depths $D \in \{1, 2, 3, 4, 8\}$ and the four
    decoding configurations. \emph{Direct} uses MRP predictions without
    verification; \emph{Speculative} uses the backbone to verify and
    correct MRP drafts.}
    \label{tab:layer-sweep-1p7b}
\end{table}

\subsection{Proof of Corollary~\ref{cor:decay}}

\begin{proof}
The bound in Proposition~\ref{prop:contraction} tightens as denoising progresses through two mechanisms. First, the number of tokens revealed per step $|\mathcal{R}_t|$ is bounded by the number of remaining masked positions $m_t$, which decreases monotonically along the chain. In threshold-based decoding, later steps reveal fewer tokens because the remaining positions are those where the backbone had low confidence; in fixed-schedule decoding, $|\mathcal{R}_t|$ is typically non-increasing by design. Second, as more context is revealed, predictions at remaining positions become more concentrated --- when $\max_v \pi_t^{i,v} \geq 1 - \eta$ for small $\eta$, no perturbation can shift the distribution by more than $O(\eta)$ in total variation, regardless of the Lipschitz constant.

Combining these, the average per-position TV distance at step $t$ satisfies:
\begin{equation}
    \frac{1}{m_{t-1}} \sum_{i \notin \mathcal{R}_t} D_{\mathrm{TV}}(\*\pi_{t-1}^i, \*\pi_t^i) \leq \kappa D_{\max} \cdot \frac{m_t}{L},
\end{equation}
where $D_{\max} = \max_v \|\*e_v - \*e_{\texttt{[MASK]}}\|_2$. Since $m_t/L$ decreases monotonically from $1$ to $0$, the residual magnitude decays along the chain.
\end{proof}

\end{document}